\documentclass[3p,review]{elsarticle}

\usepackage{graphicx} 
\usepackage{amsmath, amsfonts, amssymb, mathtools}  

\DeclareMathOperator*{\argmin}{argmin}
\DeclarePairedDelimiter{\ceil}{\lceil}{\rceil}
\usepackage{standalone}
\newcommand*{\tran}{^\top}  
\usepackage{multirow} 

\usepackage[skip=0pt]{caption}

\usepackage{lineno, hyperref}
\modulolinenumbers[5]
\makeatletter
\def\ps@pprintTitle{%
 \let\@oddhead\@empty
 \let\@evenhead\@empty
 \def\@oddfoot{\footnotesize\itshape arXiv post-print. This article was published in: doi.org/10.1016/j.ijforecast.2020.06.003 \hfill\today}%
 \let\@evenfoot\@oddfoot}
\makeatother
\biboptions{authoryear}

\usepackage{caption}
\usepackage{subcaption}

\usepackage{tikz}
\usetikzlibrary{shapes, positioning, fit}
\tikzset{
    block/.style = {draw, fill=white, rectangle, minimum height=1.3em, minimum width=3em},
    tmp/.style  = {coordinate}, 
    sum/.style= {draw, fill=white, circle, node distance=1cm},
    input/.style = {coordinate},
    output/.style= {coordinate},
    pinstyle/.style = {pin edge={to-,thin,black}},
    bolanro/.style = {anchor=center,circle,fill=black,text=white,font=\scriptsize,inner sep=0.25mm}
}

\usepackage{comment}
\usepackage{soul}
\usepackage{todonotes}

\definecolor{fgreen}{rgb}{0.0, 0.5, 0.0}

\usepackage{verbatim}

\newtheorem{proposition}{Proposition}

\newdefinition{rmk}{Remark}
\newproof{proof}{Proof}
\newproof{pot}{Proof of Theorem \ref{thm2}}

\begin{document}

\begin{frontmatter}
\title{A critical overview of privacy-preserving approaches \\for collaborative forecasting}


\author[inesc,fcup]{Carla Gon\c{c}alves}
\ead{carla.s.goncalves@inesctec.pt}
\author[inesc]{Ricardo J. Bessa\corref{corr}}
\ead{ricardo.j.bessa@inesctec.pt}
\author[dtu]{Pierre Pinson}
\ead{ppin@elektro.dtu.dk}

\address[inesc]{INESC TEC, Porto, Portugal}
\address[fcup]{Faculty of Sciences of the University of Porto, Portugal}
\address[dtu]{Technical University of Denmark, Copenhagen, Denmark}

\begin{abstract}
Cooperation between different data owners may lead to an improvement in forecast quality -- for instance by benefiting from spatial-temporal dependencies in geographically distributed time series. Due to business competitive factors and personal data protection questions, said data owners might be unwilling to share their data, which increases the interest in collaborative privacy-preserving forecasting. This paper analyses the state-of-the-art and unveils several shortcomings of existing methods in guaranteeing data privacy when employing Vector Autoregressive~(VAR) models. The methods are divided into three groups: data transformation, secure multi-party computations, and decomposition methods. The analysis shows that state-of-the-art techniques have limitations in preserving data privacy, such as {{\it (i)}~the necessary trade-off between privacy and forecasting accuracy, empirically evaluated through simulation and real-world experiment based on solar data; {\it (ii)}~the iterative model fitting processes which reveal data after a number of iterations.}
\end{abstract}

\end{frontmatter}


\section{Introduction}

The progress of the internet-of-things (IoT) and big data technologies is fostering a disruptive evolution in the development of innovative data analytics methods and algorithms. This {also} yields ideal conditions for data-driven services (from descriptive to prescriptive analysis), in which the accessibility to large volumes of data is a fundamental requirement. In this sense, the combination of data from different owners can provide valuable information for end-users and increase their competitiveness. 

In order to combine data coming from different sources, several statistical approaches have emerged. For example, in time series collaborative forecasting, the Vector Autoregressive (VAR) model has been widely used to forecast variables that may have different data owners. In the energy sector, the VAR model is deemed appropriate to update very short-term forecasts (e.g., from 15 minutes to 6 hours ahead) with recent data, thus taking advantage of geographically distributed data collected from sensors (e.g., anemometer, pyranometer) and/or wind turbines and solar power inverters~\citep{tastu2013probabilistic, Bessa2015}. The VAR model can also be used in short-term electricity price forecasting~\citep{Ziel2018}.
Furthermore, the large number of potential data owners favors the estimation of the VAR model's coefficients by applying distributed optimization algorithms. The Alternating Direction Method of Multipliers (ADMM) is a widely used convex optimization technique; see \cite{boyd2011distributed}. The combination of the VAR model and ADMM can be used jointly for collaborative forecasting \citep{Cavalcante2017}, which consists of collecting and combining information from diverse owners. The collaborative forecasting methods require sharing data or coefficients, depending on the structure of the data, and may or may not be focused on data privacy. This process is also called federated learning~\citep{yang2019federated}.

Some other examples of collaborative forecasting include: (a)~forecasting and inventory control in supply chains, in which the benefits of various types of information-sharing options are investigated~ \citep{Aviv2003, Aviv2007}; (b)~forecasting traffic flow data (i.e.\ speeds) among different locations~\citep{Ravi2009}; (c)~forecasting retail prices of a specific product at every outlet by using historical retail prices of the product at a target outlet and at competing outlets~\citep{Ahmad2016}. The VAR model is the simplest collaborative model but, conceptually, a collaborative forecasting model for time series does not need to be a VAR. Furthermore, it is possible to extend the VAR model to include exogenous information (see~\cite{Nicholson2017} for more details) and to model non-linear relationships with past values (e.g. \cite{Li2009} extend the additive model structure to a multivariate setting). 

Setting aside the significant potential of the VAR model for collaborative forecasting, the concerns with the privacy of personal and commercially sensitive data comprise a critical barrier and require privacy-preserving algorithmic solutions for estimating the coefficients of the model. 

A confidentiality breach occurs when third parties recover the data provided in confidence without consent. The leaking of a single record from the dataset may have a different impact according to the nature of the data. For example, in medical data where each record represents a different patient, it could lead to the disclosure of all the details about said patient. On the other hand, concerning renewable energy generation time series, knowing that 30~MWh was produced in a given hour does not provide very relevant information to a competitor. Hereafter, the term confidentiality breach designates the reconstruction of the entire dataset by another party.

These concerns with data confidentiality motivated the research to handle confidential data in methods such as linear regression and classification problems~\citep{du2004privacy}, ridge linear regression~\citep{karr2009privacy}, logistic regression~\citep{wu2012g}, survival analysis~\citep{lu2015webdisco}, aggregated statistics for time series data~\citep{Jia2014}, etc. Aggregated statistics consist of aggregating a set of time series data through a specific function, such as the average (e.g., average amount of daily exercise), sum, minimum or maximum. However, certain literature approaches identified confidentiality breaches, showing that the statistical methods developed to protect data privacy should be analyzed to confirm their robustness and {that} additional research may be required to address overlooked limitations~\citep{fienberg2009valid}. Furthermore, the application of these methodologies to the VAR model needs to be carefully analyzed, since the target variables are the time series of each data owner, and the covariates are the lags of the same time series, meaning that both target and covariates share a large proportion of values. 

The simplest solution would be {having the} data owners agreeing on a commonly trusted entity (or a central node), {capable of gathering} private data, solving the associated model's fitting problem on behalf of the data owners, and then returning the results~\citep{pinson2016introducing}. However, in many cases, the data owners are unwilling to share their data even with a trusted central node. This has also motivated the development of data markets to monetize data and promote data sharing~\citep{Agarwal2018}, which can be driven by blockchain and smart contracts technology~\citep{Kurtulmus2018}.

Another possibility would be to apply differential privacy mechanisms, which consist in adding properly calibrated noise to an algorithm (e.g., adding noise to the coefficients estimated during each iteration of the fitting procedure) or to the data itself. Differential privacy is not an algorithm, but a rigorous definition of privacy that is useful for quantifying and bounding privacy loss (i.e., how much original data a party can recover when receiving data protected with added noise)~\citep{dwork2009differential}. It requires computations insensitive to changes in any particular record or intermediate computations, thereby restricting data leaks through the results -- this is elaborated in~\ref{diff-privacy}. While computationally efficient and popular, these techniques invariably degrade the predictive performance of the model~\citep{yang2019federated} and are not very effective, as this paper shows.

The present paper conducts a review of the state-of-the-art in statistical methods for collaborative forecasting with privacy-preserving approaches. This work is not restricted to a simple overview of the existing methods and {it} performs a critical evaluation of said methods, from a mathematical and numerical point of view, namely when applied to the VAR model. The major contribution to the literature is to show gaps and downsides of current methods and to present insights for further improvements towards fully privacy-preserving VAR forecasting methods.   

In this work, we analyze the existing state-of-the-art privacy-preserving techniques, dividing them into the following groups: 
\begin{itemize}
\item \textit{Data transformation methods}: each data owner transforms their data before the model's fitting process, by adding randomness to the original data in such a way that high accuracy and privacy can be achieved at the end of the fitting process. The statistical method is independent of the transformation function and it is applied to the transformed data.
\item \textit{Secure multi-party computation protocols}: the encryption of the data occurs while fitting the statistical model (i.e.\ intermediate calculations of an iterative process) and data owners are required to conjointly compute a function over their data with protocols for secure matrix operations. A protocol consists of a set of rules that determine how data owners must operate to determine said function. These rules establish the calculations assigned to each data owner, what information should be shared among them, in addition to the conditions necessary for the adequate implementation of said calculations.
\item \textit{Decomposition-based methods}: the optimization problem is decomposed into sub-problems allowing each data owner to fit their model's coefficients separately.
\end{itemize}

The remaining of the paper is organized as follows: {S}ection~\ref{sec:privacy_methods} describes the state-of-the-art for collaborative forecasting with privacy-preserving; {S}ection~\ref{sec:background} describes the VAR model as well as coefficients estimators, and performs a critical evaluation of the state-of-the-art methods when applied to the VAR model (solar energy time series data are used in the numerical analysis). Section~\ref{sec:discussion} focuses on the discussion and comparison of the presented approaches, while the conclusions are presented in {S}ection~\ref{sec:conclusion}.

\section{Privacy-preserving Approaches}\label{sec:privacy_methods}

For notation purposes, vectors and matrices are denoted by bold small and capital letters, e.g. $\mathbf a$ and $\mathbf A$, respectively. The vector $\mathbf a=[a_1, \dots, a_k]^\top$ represents a column vector with $k$ dimension, where $a_i$ are scalars, $i=1,\dots,k$. The column-wise joining of vectors and matrices is indicated by $[\mathbf a, \mathbf b]$ and $[\mathbf A, \mathbf B]$, respectively. 

Furthermore, $\mathbf Z \in \mathbb{R}^{T \times M}$ is the covariate matrix and $\mathbf Y \in \mathbb{R}^{T\times N}$ is the target matrix, considering $n$ data owners. The values $T$, $M$ and $N$ are the number of records, covariates and target variables, respectively. When considering collaborative forecasting models, different divisions of the data may be considered. Figure~\ref{fig:datasplit_} shows the most common one, i.e.
\begin{enumerate}
\item \textit{Data split by records:} the data owners, represented as $A_i, i=1,\dots, n$, observe the same features for different groups of samples, e.g. different timestamps in the case of time series. $\mathbf Z$ is split into $\mathbf Z_{A_i}^r \in \mathbb{R}^{T_{A_i} \times M}$ and $\mathbf Y$ into $\mathbf Y_{A_i}^r \in \mathbb{R}^{T_{A_i} \times N}$, such that $\sum_{i=1}^n T_{A_i}=T$; 
\item \textit{Data split by features:} the data owners observe different features of the same {record}s. $\mathbf Z = [\mathbf Z_{A_1},\dots,\mathbf Z_{A_n}]$, $\mathbf Y = [\mathbf Y_{A_1},\dots,\mathbf Y_{A_n}]$, such that $\mathbf Z_{A_i} \in \mathbb{R}^{T \times M_{A_i}}$, $\mathbf Y_{A_i} \in \mathbb{R}^{T \times N_{A_i}}$, with $\sum_{i=1}^n M_{A_i}=M$ and $\sum_{i=1}^n N_{A_i}=N$; 
\end{enumerate}

This section summarizes state-of-the-art approaches to deal with privacy-preserving collaborative forecasting methods. Section~\ref{sec:TBM} describes the methods that ensure confidentiality by transforming the data. Section~\ref{sec:SCM} presents and analyzes the secure multi-party protocols. Section~\ref{sec:decomp} describes the decomposition-based methods.

\begin{figure}
\centering
\includegraphics[trim={0 0 0 0cm},clip]{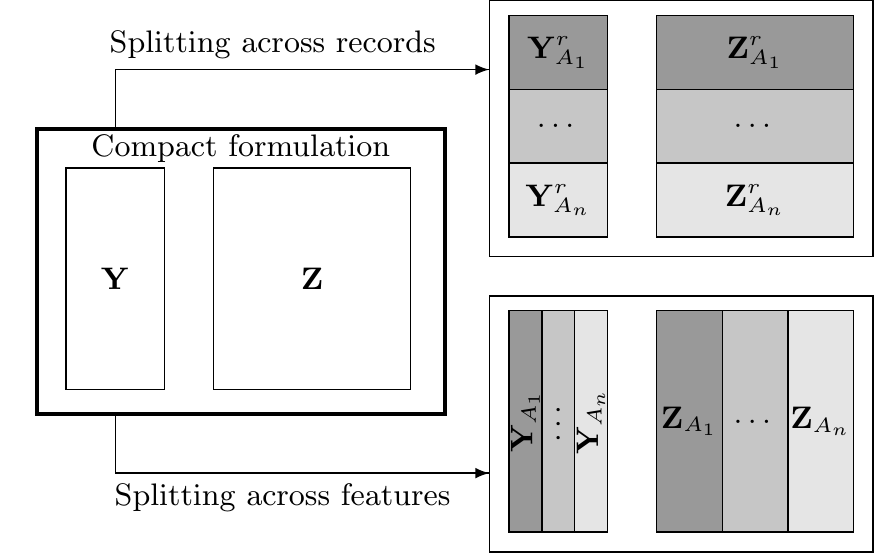}
\caption{Common data division structures.}
\label{fig:datasplit_}
\end{figure}

\subsection{Data Transformation Methods}\label{sec:TBM}

Data transformation methods use operator $\mathcal T$ to transform the data matrix $\mathbf X$ into $\tilde{\mathbf X}=\mathcal T(\mathbf X)$. Then, the problem is solved in the transformed domain. A common method of {masking} sensitive data is adding or multiplying it by perturbation matrices. In additive randomization, random noise is added to the data in order to mask the values of records. Consequently, the more masked the data becomes, the more secure it will be, as long as the differential privacy definition is respected (see~\ref{diff-privacy}). However, the use of randomized data implies the deterioration of the estimated statistical models, {and the estimated coefficients of said data} should be close to the estimated {coefficients after} using original data~\citep{zhou2009compressed}. 

Concerning multiplicative randomization, it {enables} changing the dimensions of the data by multiplying it by random perturbation matrices. If the perturbation matrix $\mathbf W \in \mathbb{R}^{k\times m}$ multiplies the original data $\mathbf X \in \mathbb{R}^{m\times n}$ on the left (pre-multiplication), i.e.\  $\mathbf{WX}$, then it is possible to change the number of records; otherwise, if $\mathbf W\in \mathbb{R}^{n\times s}$ multiplies $\mathbf X\in \mathbb{R}^{m\times n}$ on the right (post-multiplication), i.e.\  $\mathbf{XW}$, it is possible to modify the number of features. Hence, it is possible to change both dimensions by applying both pre and post-multiplication by perturbation matrices.

\subsubsection{Single Data Owner}

The use of linear algebra to mask the data is a common practice in recent outsourcing approaches, in which a data owner resorts to the cloud to fit their model's coefficients, without sharing confidential data. For example, in~\cite{ma2017efficient} the coefficients that optimize the linear regression model
\begin{equation}
    \mathbf{y}=\mathbf{X} \boldsymbol{\beta} + \boldsymbol{\varepsilon} \ ,
\end{equation}
with covariate matrix $\mathbf{X} \in \mathbb{R}^{m\times n}$, target variable $\mathbf{y} \in \mathbb{R}^n$, coefficient vector $\boldsymbol{\beta}\in \mathbb{R}^n$ and error vector $\boldsymbol{\varepsilon} \in \mathbb{R}^n$, are estimated through the regularized least squares estimate for the ridge {linear} regression, with penalization term $\lambda>0$,
\begin{equation}
    \hat{\boldsymbol{\beta}}_{\text{ridge}} = (\mathbf{X}\tran\mathbf{X}+\lambda \mathbf{I})^{-1}\mathbf{X}\tran\mathbf{y}.
\end{equation}
In order to compute $\hat{\boldsymbol{\beta}}_{\text{ridge}}$ {via} a cloud server, the authors consider that 
\begin{equation}
    \hat{\boldsymbol{\beta}}_{\text{ridge}} = \mathbf{A}^{-1} \mathbf{b} \ ,
\end{equation} 
where $\mathbf{A}=(\mathbf{X}\tran\mathbf{X}+\lambda \mathbf{I})^{-1}$ and $\mathbf{b}=\mathbf{X}\tran\mathbf{y}$, $\mathbf{A} \in \mathbb{R}^{n \times n}$, $\mathbf{b} \in \mathbb{R}^{n}$. Then, the masked matrices $\mathbf{MAN}$  and $ \mathbf{M}(\mathbf{b}+\mathbf{Ar})$ are sent to the server which computes
\begin{equation}\label{eq:transf_system}
    \hat{\boldsymbol{\beta}}' = (\textbf{MAN})^{-1}(\mathbf{M}(\mathbf{b}+\mathbf{Ar})) \ ,
\end{equation} 
where $\mathbf M$, $\mathbf N$, and $\mathbf r$ are randomly generated matrices, $\mathbf{M}, \mathbf{N} \in \mathbb{R}^{n\times n}$, $\mathbf{r} \in \mathbb{R}^{n}$. Finally, the data owner receives $\hat{\boldsymbol{\beta}}'$ and recovers the original {coefficients} by computing $\hat{\boldsymbol{\beta}}_\text{ridge}= \mathbf{N}\hat{\boldsymbol{\beta}}'-\mathbf{r}$.

Data normalization is also a data transformation approach that masks data by transforming the original features into a new range through the use of a mathematical function. There are many methods for data normalization, the most important ones being $z$-score and min-max normalization~\citep{jain2011min}, which are useful when the actual minimum and maximum values of the features are unknown. However, in many applications, these values are either known or publicly available, and normalized values still encompass commercially valuable information.

{As to} time series data, other approaches for data randomization make use of the Fourier and wavelet transforms. The Fourier transform allows representing periodic time series as a linear combination of sinusoidal components (sine and cosine). In~\cite{papadimitriou2007time}, each data owner generates a noise time series by: (i)~adding Gaussian noise to relevant coefficients, or (ii)~{disrupting} each sinusoidal component by randomly changing its magnitude and phase. Similarly, the wavelet transform represents the time series as a combination of functions (e.g. the Mexican hat or the Poisson wavelets), and randomness can be introduced by adding random noise to the coefficients~\citep{papadimitriou2007time}. However, there are no privacy guarantees since noise does not respect any formal definition such as differential privacy. 

\subsubsection{Multiple Data Owners} \label{sec:post_multiplication_multiple}

The task of masking data is even more challenging when dealing with different data owners, since it is crucial to ensure that the transformations that data owners make to their data preserve the real relationship between the variables or the time series.

Usually, for generalized linear models (linear regression model, logistic regression model, among others), where $n$ data owners observe the same features, i.e. data is split by records as illustrated in Figure~\ref{fig:datasplit_}, each data owner $A_i,  i = 1,...,n$, can individually multiply their covariate matrix $\mathbf Z^r_{A_i}\in \mathbb{R}^{{T_{A_i}\times M}}$ and target variable $\mathbf Y^r_{A_i} \in \mathbb{R}^{T_{A_i}\times N}$ by a random matrix $\mathbf M_{A_i} \in \mathbb{R}^{k\times T_{A_i}}$ (with a jointly defined $k$ value), providing $\mathbf M_{A_i} \mathbf Z^r_{A_i},\allowbreak \mathbf M_{A_i} \mathbf Y^r_{A_i}$ to the competitors~\citep{mangasarian2012privacy,yu2008privacy}, which allows pre-multiplying the original data{,} 
\begin{equation*}
\mathbf Z^r=
\left[\begin{array}{c}
\mathbf Z^r_{A_1}\\
\vdots\\
\mathbf Z^r_{A_n}\\
\end{array}
\right] \text{ and }
\mathbf Y^r=
\left[\begin{array}{c}
\mathbf Y^r_{A_1}\\
\vdots\\
\mathbf Y^r_{A_n}\\
\end{array}
\right]{\ ,} 
\end{equation*}
by $\mathbf M=\allowbreak [\mathbf M_{A_1},\allowbreak \dots, \allowbreak \mathbf M_{A_n}]$, since 
\begin{equation}
\mathbf{MZ}^r=\mathbf{M}_{A_1} \mathbf{Z}^r_{A_1}+\dots+ \mathbf{M}_{A_n} \mathbf{Z}^r_{A_n} \ .
\end{equation}
The same holds for the multiplication $\mathbf{MY}^r$, $\mathbf{M} \in \mathbb{R}^{k \times \sum_{i=1}^n T_{A_i}}, \allowbreak \mathbf Z^r \in \mathbb{R}^{\sum_{i=1}^n T_{A_i}\times M}, \allowbreak \mathbf Y^r \in \mathbb{R}^{\sum_{i=1}^n T_{A_i}\times N}$.
{
This definition of $\mathbf M$ is possible because when multiplying $\mathbf M$ and $\mathbf Z^r$, the $j$-th column of $\mathbf M$ only multiplies the $j$-th row of $\mathbf Z^r$.} For some statistical learning algorithms, a property of such matrix is the orthogonality, i.e. $\mathbf M^{-1}=\mathbf M\tran$.
Model fitting is then performed with this new representation of the data, which preserves the solution to the problem. This is true in the linear regression model because the {multivariate} least squares {estimate} for the linear regression model with covariate matrix $\mathbf{MZ}^r$ and target variable $\mathbf{MY}^r$ is 
\begin{equation}
\hat{\mathbf{B}}_\text{LS}= \left((\mathbf{Z}^r)\tran\mathbf{Z}^r\right)^{-1}\left((\mathbf{Z}^r)\tran\mathbf{Y}^r\right) \ , 
\end{equation}
which is also the multivariate least squares {estimate} for the coefficients of a linear regression considering data matrices $\mathbf{Z}^r$ and $\mathbf{Y}^r$, respectively. Despite this property, the application in LASSO regression does not guarantee that the sparsity of the coefficients is preserved and a careful analysis is necessary to ensure the correct estimation of the model~\citep{zhou2009compressed}. \cite{liu2008survey} discuss attacks based on prior knowledge, in which a data owner estimates $\mathbf M$ by knowing a small collection of original data records.
Furthermore, when considering the {linear regression} model for which $\mathbf Z =[\mathbf Z_{A_1}, \dots, \mathbf Z_{A_n}]$ and $\mathbf Y =[\mathbf Y_{A_1}, \dots, \mathbf Y_{A_n}]$, {i.e. data is split by features,} it is not possible to define a matrix $\mathbf M^*=[\mathbf{M}^*_{A_1},\dots, \mathbf{M}^*_{A_n}]\in \mathbb{R}^{k \times T}$ and then privately compute $\mathbf{M^*Z}$ and $\mathbf{M^*Y}$, because as explained, the $j$-th column of $\mathbf M^*$ multiplies the $j$-th row of $\mathbf Z$, which, in this case, consists of data coming from different owners.

Similarly, if the data owners observe different features, a linear programming problem can be solved in a way that each {individual} data owner multiplies their data $\mathbf X_{A_i} \in \mathbb{R}^{T\times M_{A_i}}$ by a private random matrix $\mathbf N_{A_i} \in \mathbb{R}^{M_{A_i}\times s}$ (with a jointly defined value $s$) and, then, shares $\mathbf X_{A_i} \mathbf N_{A_i}$~\citep{mangasarian2011privacy}, $i=1,...,n$, which is equivalent to post-multiplying the original dataset $\mathbf X=\allowbreak[\mathbf X_{A_1},\allowbreak...,\allowbreak \mathbf X_{A_n}]$ by $\mathbf N=\allowbreak[\mathbf N\tran_{A_1},\allowbreak \dots, \allowbreak\mathbf N\tran_{A_n}]\tran$, which {represents} the joining of $\mathbf N_{A_i}, i=1,\dots,n$, through row-wise operation. However, the obtained solution is in a different space, and {it} needs to be recovered by multiplying it by the corresponding $\mathbf N_{A_i}, i=1,..., n$. 
For the {linear regression}, which models the relationship between the covariates $\mathbf{Z}\in\mathbb{R}^{T\times M}$ and the target $\mathbf{Y}\in\mathbb{R}^{T\times N}$, this algorithm corresponds to solving a linear regression that models the relationship between $\mathbf{ZN_z}$ and $\mathbf{YN_y}$, i.e.\ {the solution is given by}
\begin{equation} \label{eq:post-var}
	\hat{\mathbf B}'_\text{LS}=\argmin_{\mathbf B} \left(\frac{1}{2}\|\mathbf{YN_y}-\mathbf{ZN_zB}\|_2^2\right),
\end{equation}
where $\mathbf{ZN_z}$ and $\mathbf{YN_y}$ are shared matrices. Two private matrices $\mathbf{N_z} \in \mathbb R^{M\times s}$, $\mathbf{N_y} \in \mathbb R^{N\times w}$ are required to transform the data, since the number of columns for $\mathbf Z$ and $\mathbf Y$ is different ($s$ and $w$ values are jointly defined). The problem is that the {multivariate} least squares {estimate} for~\eqref{eq:post-var} is {given by}
\begin{equation}\label{eq:post_multipication}
    \hat{\mathbf B}'_\text{LS}=\Big((\mathbf{ZN_z})\tran(\mathbf{ZN_z})\Big)^{-1}\Big((\mathbf{ZN_z})\tran(\mathbf{YN_y})\Big)= (\mathbf{N_z})^{-1}\underbrace{(\mathbf{Z\tran Z})^{-1}\mathbf{Z\tran Y}}_\text{$=\argmin_{\mathbf B} \left(\frac{1}{2}\|\mathbf{Y}-\mathbf{ZB}\|_2^2\right)$}\mathbf{N_y} \ , 
\end{equation}
which implies that this transformation does not preserve the {coefficients} of the {linear regression considering data matrices $\mathbf{Z}$ and $\mathbf{Y}$, respectively, and therefore $\mathbf N_z$ and $\mathbf N_y$ would have to be shared.}

Generally, data transformation is performed through the generation of random matrices that pre- or post-multiply the private data. However, there are other techniques through which data is transformed with matrices defined according to the data itself, as is the case with Principal Component Analysis (PCA). PCA
is a widely used statistical procedure for reducing the dimension of data, by applying an orthogonal transformation that retains as much of the {data} variance as possible. Considering the matrix $\mathbf W \in \mathbb{R}^{M \times M}$ of the eigenvectors of the covariance matrix $\mathbf Z\tran \mathbf Z$, $\mathbf{Z}\in \mathbb{R}^{T\times M}$, PCA allows representing the data by $L$ variables performing $\mathbf Z \mathbf N_L$, where $\mathbf N_L$ are the first $L$ columns of $\mathbf W$, $L=1,..., M$.  For the data split by records, \cite{dwork2014analyze} suggest a differential private PCA, assuming that each data owner takes a random sample of the fitting records to form the covariate matrix. {In order to protect} the covariance matrix, one can add Gaussian noise to this matrix (determined without sensible data sharing), leading to the computation of the principal directions of the noisy covariance matrix. To finalize the process, the data owners multiply the sensible data by {said} principal directions  before feeding it into the model fitting. Nevertheless, the application to collaborative {linear regression with data split by features} would require sharing the data when computing the $\mathbf Z\tran \mathbf Z$ matrix, since $\mathbf Z\tran$ is divided by rows. Furthermore, as explained in~\eqref{eq:post-var} and~\eqref{eq:post_multipication}, it is difficult to recover the original {linear} regression model by performing the estimation of the coefficients using transformed covariates and target matrices, through post-multiplication by random matrices.

{Regarding} the data normalization techniques mentioned above, \cite{zhu2015privacy} assume that data owners mask their data by using $z$-score normalization, followed by the sum of random noise (from Uniform or Gaussian distributions), allowing a greater control on their data, which is {then} shared with a recommendation system {that} fits the model. However, the noise does not {meet} the differential privacy definition~(see~\ref{diff-privacy}). 

For data collected by different sensors (e.g., smart meters and mobile users) it is common to proceed to the aggregation of data through privacy-preserving techniques. For instance, by adding carefully calibrated Laplacian noise to each time series~\citep{fan2014adaptive,soria2017individual}. The addition of noise to the data is an appealing technique given its easy application. However, even if this noise meets the definition of differential privacy, there is no guarantee that the resulting model will perform well.

\subsection{Secure Multi-party Computation Protocols}\label{sec:SCM}

In secure multi-party computation, the intermediate calculations required by the fitting algorithms, which require data owners to jointly compute a function over their data, are performed through protocols for secure operations, such as matrix addition or multiplication (discussed in {S}ection~\ref{linear-algebra-protocols}). In these approaches, the encryption of the data occurs while fitting the model (discussed in {S}ection~\ref{sec:cryp}), instead of as a pre-processing step such as in the data transformation methods from the previous section.

\subsubsection{Linear Algebra-based Protocols}\label{linear-algebra-protocols}

The simplest secure multi-party computation protocols are based on linear algebra and address the problems where matrix operations with confidential data are necessary.
\cite{du2004privacy} propose  secure protocols for product $\mathbf A. \mathbf C$ and inverse of the sum $(\mathbf A+\mathbf C)^{-1}$, for any two private matrices $\mathbf A$ and $\mathbf C$ with appropriate dimensions. The aim is to fit a (ridge) linear regression between two data owners, who observe different covariates but share the target variable. Essentially, the $\mathbf A. \mathbf C$ protocol transforms the product of matrices, $\mathbf{A} \in \mathbb{R}^{m \times s}$, $\mathbf C \in \mathbb{R}^{s \times k}$, into a sum of matrices, $\mathbf V_a+\allowbreak\mathbf V_c$, which are equally secret, $\mathbf V_a, \mathbf V_c \in \mathbb{R}^{m \times k}$. However, since the estimate of the coefficients for linear regression with covariate matrix $\mathbf Z {\in \mathbb{R}^{T\times M}}$ and target {matrix} $\mathbf Y {\in \mathbb{R}^{T \times N}}$ is 
\begin{equation}
\hat{ \mathbf B}_\text{LS}=(\mathbf Z\tran \mathbf Z)^{-1} \mathbf Z\tran \mathbf Y \ ,    
\end{equation} 
the $\mathbf{A.C}$ protocol is used to perform the computation of $\mathbf{V}_a, \mathbf{V}_c$ such that 
\begin{equation}
    \mathbf V_a +\allowbreak\mathbf V_c = (\mathbf{Z\tran Z}) \ ,
\end{equation} 
{which requires} the definition of an $(\mathbf A+\mathbf C)^{-1}$ protocol to compute 
\begin{equation}
    (\mathbf{Z\tran Z} )^{-1}=(\mathbf V_a +\allowbreak\mathbf V_c)^{-1}.
\end{equation} 

For the $\mathbf A.\mathbf C$ protocol, $\mathbf{A} \in \mathbb{R}^{m \times s}$, $\mathbf C \in \mathbb{R}^{s \times k}$, {there are} two different formulations, according to the existence, or not, of a third entity. {In cases where only two data owners perform the protocol}, a random matrix $\mathbf M\in \mathbb{R}^{s\times s}$ is jointly generated and the $\mathbf A. \mathbf C$ protocol achieves the following results, by dividing the $\mathbf M$ and $\mathbf M^{-1}$ into two matrices with the same dimensions 
\begin{align}
\mathbf{A}\mathbf C &= \mathbf{AMM^{-1}}\mathbf C = \mathbf A[\mathbf M_\text{left}, \mathbf M_\text{right}]
\left[\begin{tabular}{l}
$(\mathbf M^{-1})_\text{top}$ \\
$(\mathbf M^{-1})_\text{bottom}$
\end{tabular}\right]
\mathbf C \\
& = \mathbf{AM}_\text{left}(\mathbf M^{-1})_\text{top}\mathbf C + \mathbf{AM}_\text{right}(\mathbf M^{-1})_\text{bottom}\mathbf C \ ,
\end{align}
where $\mathbf M_\text{left}$ and $\mathbf M_\text{right}$ {represent} the left and right part of $\mathbf M$, and $(\mathbf M^{-1})_\text{top}$ and $(\mathbf M^{-1})_\text{bottom}$ {designate} the top and bottom part of $\mathbf M^{-1}$, respectively.
In this case, 
\begin{equation}\label{eq:VaProtocol1}
    \mathbf V_a =\mathbf{AM}_\text{left} (\mathbf M^{-1})_\text{top}\mathbf C \ ,
\end{equation}
is derived by the first data owner, and 
\begin{equation}\label{eq:VcProtocol1}
    \mathbf V_c=\allowbreak \mathbf{AM}_\text{right}\allowbreak  (\mathbf M^{-1})_\text{bottom}\mathbf C \ ,
\end{equation} 
by the second one. Otherwise, a third entity is assumed to generate random matrices $\mathbf R_a,\mathbf r_a$ and $\mathbf R_c,\mathbf r_c$, such that 
\begin{equation}
    \mathbf r_a + \mathbf r_c =\mathbf R_a \mathbf R_c \ ,
\end{equation} 
which are sent to the first and second data owners, respectively, $\mathbf R_a \in \mathbb{R}^{m\times s}$, $\mathbf R_c \in \mathbb{R}^{s\times k}$, $\mathbf r_a, \mathbf r_c \in \mathbb{R}^{m\times k}$. In this case, the data owners start by trading the matrices $\mathbf{A+R}_a$ and $\mathbf{C+ R}_c$, then the second data owner randomly generates a matrix $\mathbf V_c$ and sends 
\begin{equation}
    \mathbf T=(\mathbf{A+R}_a)\mathbf C+(\mathbf r_c-\mathbf V_c) \ ,
\end{equation}
to the first data owner, in such a way that, at the end of the $\mathbf{A.C}$ protocol, the first data owner keeps the information 
\begin{equation}
    \mathbf V_a=\mathbf{T+r}_a-\mathbf R_a(\mathbf C+\mathbf{R}_c) \ ,
\end{equation} 
and the second keeps $\mathbf V_c$ (since the sum of $\mathbf V_a$ with $\allowbreak \mathbf V_c$ is $\allowbreak\mathbf{A}\mathbf C$). 

Finally, the $\mathbf{(A+C)^{-1}}$ protocol considers two steps, where $\mathbf{A}, \mathbf C \in \mathbb{R}^{m\times k}$. {Initially}, {the} matrix $\mathbf{(A + C)}$ is jointly converted to $\mathbf{P(A + C)Q}$ using two random matrices, $\mathbf P$ and $\mathbf Q$, {which} are only known to the second data owner to prevent the first one from learning matrix $\mathbf C$, $\mathbf P \in \mathbb{R}^{r\times m}, \mathbf Q \in \mathbb{R}^{k\times t}$. The results of $\mathbf{P(A+C)Q}$ are known only by the first data owner who can conduct the inverse computation $\mathbf{Q}^{-1}(\mathbf{A}+\mathbf C)^{-1}\mathbf{P}^{-1}$. In the {following} step, the data owners jointly remove $\mathbf Q^{-1}$ and $\mathbf P^{-1}$ and get $\mathbf{(A + C)^{-1}}$. Both steps {can be} achieved by applying the $\mathbf{A.}\mathbf C$ protocol. Although these protocols prove to be an efficient technique to solve problems with a shared target variable, {one cannot say the same when} $\mathbf Y$ is private, as {further elaborated} in {S}ubsection~\ref{subsec:emp_analysis}.

Another example of secure protocols for producing private matrices can be found in~\cite{karr2009privacy}; they are applied to data from multiple owners who observe different covariates and target features -- which are also assumed to be secret. The proposed protocol allows two data owners, with correspondent data matrix $\mathbf A$ and $\mathbf C$, $\mathbf{A} \in \mathbb{R}^{m \times k}$, $\mathbf C \in \mathbb{R}^{m \times s}$, to perform the multiplication $\mathbf{A\tran C}$ by: (i)~first data owner generates $\mathbf W=[\mathbf w_1,....,\mathbf w_g]$, $\mathbf{W} \in \mathbb{R}^{m\times g}$, such that 
\begin{equation}
    \mathbf w\tran_i \mathbf A_j=\mathbf 0 \ , 
\end{equation}
where $\mathbf A_j$ is the $j$-th column of $\mathbf A$ matrix, $i=1,...,g$ and $j=1,...,k$, and {then} sends $\mathbf W$ to the second owner; (ii)~the second data owner computes $\mathbf{(I-WW\tran)}\mathbf C$ and shares it, and (iii)~the first data owner performs
\begin{equation}
    \mathbf{A\tran(I-WW\tran)C}{=\mathbf{A\tran C}-\underbrace{\mathbf{A\tran WW\tran C}}_{=\mathbf 0 \text{, since } \mathbf{A\tran W=\mathbf 0}  }}=\mathbf{A\tran C} \ ,
\end{equation} 
without the possibility of recovering $\mathbf C$, since the $rank(\mathbf{(I-WW\tran)\mathbf C})=m-g$. To generate $\mathbf W$, \cite{karr2009privacy} suggest selecting $g$ columns from the $\mathbf Q$ matrix, computed by $\mathbf{QR}$ decomposition of the private matrix $\mathbf C$, {and} excluding the first $k$ columns. 
{Furthermore, the authors define the} optimal value for $g$ {according to} the number of linearly independent equations (represented by NLIE) on the other data owner's data. The second data owner obtains $\mathbf A\tran \mathbf C$ (providing $k s$ values{, since $\mathbf{A\tran C}\in \mathbb{R}^{k\times s}$}) and receives $\mathbf W$, knowing that $\mathbf A\tran \mathbf W = 0$ (which contains $k g$ values), i.e.
\begin{equation}
\text{NLIE(Owner\#1)} = ks+kg.
\end{equation}
Similarly, the first data owner receives $\mathbf A\tran \mathbf C$ (providing $k s$ values) and $\mathbf{(I-WW\tran)C}$ (providing $s(m-g)$ values since $\mathbf{(I-WW\tran)C}\in \mathbb{R}^{m\times s}$ and $rank(\mathbf W)=m-g$), i.e.
\begin{equation}
\text{NLIE(Owner\#2)} = ks+s(m-g).
\end{equation}
\cite{karr2009privacy} determines the optimal value for $g$ by assuming that both data owners equally share NLIE, so that none of the agents benefit from the order assumed when running the protocol, i.e.
\begin{equation}
|\text{NLIE(Owner\#1)}-\text{LP(Owner\#2)}|=0\ ,
\end{equation}
which allows to obtain the optimal value $g^*=\frac{sm}{k+s}$. 

An advantage {to} this approach, when compared to the one proposed by~\cite{du2004privacy}, is that $\mathbf W$ is simply generated by the first data owner, while the invertible matrix $\mathbf M$ proposed by~\cite{du2004privacy} needs to be agreed upon by both parties, which entails substantial communication costs when the number of records is high. 

\subsubsection{Homomorphic Cryptography-based Protocols}\label{sec:cryp}

The use of homomorphic encryption was successfully introduced in model fitting and it works by encrypting the original values in such a way that the application of arithmetic operations in the public space does not compromise the encryption. Homomorphic encryption ensures that, after the decryption stage (in the private space), the resulting values correspond to the ones obtained by operating on the original data. Consequently, homomorphic encryption is especially {responsive} and engaging to privacy-preserving applications. As an example, the Paillier homomorphic encryption scheme defines that (i)~two integer values encrypted with the same public key may be multiplied together to give encryption of the sum of the values, and (ii)~an encrypted value may be taken to some power, yielding encryption of the product of the values. \cite{hall2011secure} proposed  a secure protocol for summing and multiplying real numbers by extending the Paillier encryption, aiming to perform matrix products required to solve linear regression, for data divided by features or records.

Equally based in Paillier encryption, the work of~\cite{nikolaenko2013privacy} introduces {two parties that correctly perform their tasks without teaming up to discover private data}: a crypto-service provider (i.e., a party that provides software or hardware-based encryption and decryption services) and an evaluator {(i.e., a party who runs the learning algorithm)}, in order to perform a secure linear regression for data split by records. Similarly, \cite{chen2018privacy} use Paillier and ElGamal encryptions to fit the coefficients of ridge {linear} regression, also including these entities. In both works, the use of the crypto-service provider is {prompted} by assuming that the evaluator does not corrupt its computation in producing an incorrect result. 
Two conditions are required to ensure that there will be no confidentiality breaches: the crypto-service provider must publish the system keys correctly, and there can be no collusion between the evaluator and the crypto-service provider. 
The data could be reconstructed if the crypto-service provider supplies correct keys to a curious evaluator. 
For data divided by features, the work of~\cite{gascon2017privacy} extends the approach of~\cite{nikolaenko2013privacy} by designing a secure multi/two-party inner product.

{\cite{jia2018preserving} explore a privacy-preserving data classification scheme with a support vector machine, thus ensuring} that the data owners can successfully conduct data classification without exposing their learned models to the ``tester'', while the ``testers'' keep their data private. For example, a hospital (owner) can create a model to learn the relation between a set of features and the existence of a disease, and another hospital (tester) can use this model to obtain a forecasting value, without any knowledge about the model. The method is supported by cryptography-based protocols for secure computation of multivariate polynomial functions but, unfortunately, this only works for data split by records.

{\cite{li2012efficient} addresses the privacy-preserving computation of the sum and the minimum of multiple time series collected by different data owners, by combining homomorphic encryption and a novel key management technique to support large data dimensions. These statistics with a privacy-preserving solution for individual user data are quite useful to explore mobile sensing in different applications such as environmental monitoring (e.g., average level of air pollution in an area), traffic monitoring (e.g., highest moving speed during rush hour), healthcare (e.g., number of users infected by a flu), etc.} \cite{liu2018practical} and \cite{li2018ppma} explored similar approaches, based on Paillier or ElGamal encryption, namely concerning the application in smart grids. However, the estimation of models such as the linear regression model also requires protocols for the secure product of matrices.

Homomorphic cryptography was further explored to solve secure linear programming problems through intermediate steps of the simplex method, which optimizes the problem by using slack variables, tableaus, and pivot variables~\citep{de2012design}. The author observed that the proposed protocols are still unviable to solve linear programming problems, having numerous variables and constraints, which are quite reasonable in practice.

\cite{aono2017input} perform a combination of homomorphic cryptography with differential privacy, in order to deal with data split by records. Summarily, if data is split by records, as illustrated in Figure~\ref{fig:datasplit_}, each $i$-th data owner observes the covariates $\mathbf Z^r_{A_i}$ and target variable $\mathbf Y^r_{A_i}$, $\mathbf Z^r_{A_i} \in \mathbb{R}^{T_{A_i} \times M}, \mathbf Y^r_{A_i} \in \mathbb{R}^{T_{A_i} \times N},  i=1,...,n$. Then $(\mathbf Z^r_{A_i})\tran \mathbf Z^r_{A_i}$ and $(\mathbf Z^r_{A_i})\tran \mathbf Y^r_{A_i}$ are computed and Laplacian noise is added to them. This information is encrypted and sent to the cloud server, which works on the encrypted domain, summing all the matrices received. Finally, the server provides the result of this sum to a client who decrypts it and obtains relevant information to perform the linear regression, i.e.\ $\sum_{i=1}^n( \mathbf Z^r_{A_i})\tran \mathbf Z^r_{A_i}$, 
$\sum_{i=1}^n (\mathbf Z^r_{A_i})\tran \mathbf Y^r_{A_i}$, etc. However, noise addition can result in a poor estimation of the coefficients, limiting the performance of the model. Furthermore, this is not valid when data is divided by features{, because} $\mathbf Z\tran \mathbf Z \neq \sum_{i=1}^n \mathbf Z\tran_{A_i} \mathbf Z_{A_i}$ and $\mathbf Z\tran \mathbf Y \neq \sum_{i=1}^n \mathbf Z\tran_{A_i} \mathbf Y_{A_i}$.

In summary, the cryptography-based methods are usually robust to confidentiality breaches but may require a third party for keys generation, as well as external entities to perform the computations in the encrypted domain. Furthermore, the high computational complexity is a challenge when dealing with real applications~\citep{de2012design, zhao2019secure, tran2019privacy}.

\subsection{Decomposition-based Methods}
\label{sec:decomp}

In decomposition-based methods, problems are solved by breaking them up into smaller sub-problems and solving each separately, either in parallel or in sequence. Consequently, private data is naturally distributed between the data owners. However, this natural division requires sharing of intermediate information. {For that reason}, some approaches combine decomposition-based methods with data transformation or homomorphic cryptography-based methods; but {in this paper's case}, a special {emphasis} will be given to these methods in separate.

\subsubsection{ADMM Method}\label{sec:admm}
The ADMM is a powerful algorithm that circumvents problems without a closed-form solution, such as the LASSO regression, and it has proved to be efficient and well suited for distributed convex optimization, in particular for large-scale statistical problems~\citep{boyd2011distributed}. Let $E$ be a convex forecast error function, between the true values $\mathbf Y$ and the forecasted values given by the model $\mathbf{\hat Y}=f(\mathbf B,\mathbf Z)$ using a set of covariates $\mathbf Z$ and coefficients $\mathbf B$, and $R$ a convex regularization function.
The ADMM method~\citep{boyd2011distributed} {solves the optimization problem} 
\begin{equation}
\min_{\mathbf B} E(\mathbf B)+R(\mathbf B){\ ,}
\end{equation}
by splitting $\mathbf B$ into two variables ($\mathbf B$ and $\mathbf H$),
\begin{equation}
\min_{\mathbf B , \mathbf H} E(\mathbf B)+R(\mathbf H)
\text{ subject to } \mathbf A \mathbf B+\mathbf C \mathbf H=\mathbf D{\ ,}\\
\end{equation}
and using the corresponding augmented Lagrangian function formulated with dual variable $\mathbf U$, 
\begin{equation}
    L(\mathbf B,\mathbf H,\mathbf U) = E(\mathbf B)+R(\mathbf H)+\mathbf U\tran(\mathbf A\mathbf B\mathbf +\mathbf C\mathbf H-\mathbf D)+\frac{\rho}{2}\|\mathbf A\mathbf B+\mathbf C\mathbf H-\mathbf D\|_2^2.
\end{equation}
The quadratic term $\frac{\rho}{2}\|\mathbf A\mathbf B+\mathbf C\mathbf H-\mathbf D\|_2^2$ provides theoretical convergence guarantees because it is strongly convex. This implies mild assumptions on the objective function. Even if the original objective function is convex, augmented Lagrangian is strictly convex (in some cases strongly convex) \citep{boyd2011distributed}.

The ADMM solution is estimated by the following iterative system,
\begin{equation} \label{ADMM} \left
		\{
            \begin{array}{l}
              \displaystyle 
              \mathbf B^{k+1}:=\argmin_{\mathbf B} L (\mathbf B,\mathbf H^k,\mathbf U^k)  \\
              \displaystyle 
              \mathbf H^{k+1}:=\argmin_{\mathbf H}  L(\mathbf B^{k+1},\mathbf H,\mathbf U^k) \\
              \mathbf U^{k+1}:= \mathbf U^k+\rho (\mathbf A\mathbf B^{k+1}+\mathbf C\mathbf H^{k+1}-\mathbf D).
            \end{array}
          \right.
\end{equation}

For data split by records, the consensus problem splits primal variables $\mathbf B$ and separately optimizes the decomposable cost function $E(\mathbf B)=\sum_{i=1}^n E_i(\mathbf B_{A_i})$ for all the data owners under the global consensus constraints. Considering that the sub-matrix $\mathbf Z_{A_i}^r \in \mathbb R^{T_{A_i} \times M}$ of $\mathbf Z \in \mathbb R^{T \times M}$ corresponds to the local data of the $i-$th data owner, the  coefficients $\mathbf B_{A_i} \in \mathbb R^{M \times N}$ {are given by}
\begin{align} \label{ADMMconsensus}
\begin{split}
\argmin_{\boldsymbol{\Gamma}} & \sum_i E_i(\mathbf B_{A_i})+R(\mathbf H) \\
\text{ s.t. } & {\mathbf B_{A_1}- \mathbf H = \mathbf 0, \ \mathbf B_{A_2}- \mathbf H = \mathbf 0, \dots , \
\mathbf B_{A_n}- \mathbf H = \mathbf 0 \ ,}\\
\end{split}
\end{align}
{where $\boldsymbol \Gamma = \{\mathbf B_{A_1}, \dots,\mathbf B_{A_n} , \mathbf H\}$}. In this case, $E_i(\mathbf B_{A_i})$ measures the error between the true values $\mathbf Y_{A_i}^r$ and the forecasted values given by the model $\mathbf{\hat Y}_{A_i}=f(\mathbf B_{A_i},\mathbf Z_{A_i}^r)$.

For data split by features, the sharing problem splits $\mathbf Z$ into $\mathbf Z_{A_i} \in \mathbb R^{T \times M_{A_i}}$, and $\mathbf B$ into $\mathbf B_{A_i} \in \mathbb R^{M_{A_i} \times N}$. Auxiliary $\mathbf H_{A_i} \in \mathbb R^{T \times N}$ are introduced for the $i$-th data owner based on $\mathbf Z_{A_i}$ and $\mathbf B_{A_i}$. In such case, the sharing problem is formulated based on the decomposable cost function $E(\mathbf B)= E(\sum_{i=1}^n \mathbf B_{A_i})$ and $R(\mathbf B)=\sum_{i=1}^n R(\mathbf B_{A_i})$. Then, $\mathbf B_{A_i}$ are {given by}
\begin{align} \label{ADMMsharing}
\begin{split}
\argmin_{\boldsymbol\Gamma'} \ & E(\sum_i \mathbf H_{A_i})+\sum_i R(\mathbf B_{A_i})\\
\quad\text{s.t. } &{\mathbf Z_{A_1} \mathbf B_{A_1}- \mathbf H_{A_1} = \mathbf 0, \ \mathbf Z_{A_2} \mathbf B_{A_2}- \mathbf H_{A_2} = \mathbf 0, \dots, \ \mathbf Z_{A_n} \mathbf B_{A_n}- \mathbf H_{A_n} = \mathbf 0 }\ ,\\
\end{split}
\end{align}
{where $\boldsymbol\Gamma'=\{\mathbf B_{A_1}, \dots,\mathbf B_{A_n} , \mathbf H_{A_1}, \dots,\mathbf H_{A_n}\}$}. In this case, $E(\sum_{i=1}^n \mathbf H_{A_i})$ is related to the error between the true values $\mathbf Y$ and the forecasted values given by the model $\mathbf{\hat Y}=\sum_{i=1}^nf(\mathbf B_{A_i},\mathbf Z_{A_i})$.


Undeniably, ADMM provides a desirable formulation for parallel computing~\citep{dai2018privacy}. However, it is not possible to ensure continuous privacy, since ADMM requires intermediate calculations, allowing the most curious competitors to recover the data at the end of some iterations by solving non-linear equation systems~\citep{bessa2018data}. An ADMM-based distributed LASSO algorithm, in which each data owner only communicates with its neighbor to protect data privacy, is described in~\cite{mateos2010distributed}, with applications in signal processing and wireless communications. Unfortunately, this approach is only valid for the cases where data is distributed by records.

The concept of differential privacy was also explored in ADMM by introducing randomization when computing the primal variables, i.e.\  during the iterative process, each data owner estimates the corresponding coefficients and perturbs them by adding random noise~\citep{zhang2017dynamic}. However, these local randomization mechanisms can result in a non-convergent algorithm with poor performance even under moderate privacy guarantees. To address these concerns, \cite{huang2018dp} use an approximate augmented Lagrangian function and Gaussian mechanisms with time-varying variance. Nevertheless, noise addition is not {enough} to guarantee privacy, as a competitor can potentially use the results from all iterations to perform inference~\citep{zhang2018recycled}. 

\cite{zhang2019admm} have recently combined a variant of ADMM with homomorphic encryption namely in cases where data is divided by records. 
As referred by the authors, the incorporation of the proposed mechanism in decentralized optimization under data divided by features is {quite} difficult. {Whereas the algorithm} for data split by records the algorithm only requires sharing the coefficients, the exchange of coefficients in data split by features is not enough, since each data owner observes different features. The division by features requires the local estimation of $\mathbf B_{A_i}^{k+1} \in \mathbb{R}^{M_{A_i}\times N}$ by using {information related to} $\mathbf Z_{A_j} \mathbf B^k_{A_j}$, and $\mathbf Y$, meaning that, for each new iteration, an $i-$th data owner shares $T N$ new values, instead of $M_{A_i}N$ (from $\mathbf B_{A_i}^k$), $i,j=1,...,n$.

For data split by features, {\cite{zhang2018distributed} propose} a probabilistic forecasting method combining ridge linear quantile regression and ADMM. The output is a set of quantiles instead of a unique value (usually the expected value). {In this case,} the ADMM {is} applied to split the corresponding optimization problem into sub-problems, which are solved by each data owner, assuming that all the data owners communicate with a central node in an iterative process, providing intermediate results instead of private data. In fact, the authors claimed that the paper describes how wind power probabilistic forecasting with off-site information could be achieved in a privacy-preserving and distributed fashion. However, the authors did not conduct an in-depth analysis of the method, as will be shown in {S}ection~\ref{sec:background}. Furthermore, this method assumes that the {central node knows the target matrix}.

\subsubsection{Newton-Raphson Method}
ADMM is becoming a standard technique in recent {research} about distributed computing in statistical learning, but it is not the only one. For generalized linear models, the distributed optimization for model's fitting has been efficiently achieved through the Newton-Raphson method, which minimizes {a twice differentiable forecast error function $E$, between the true values $\mathbf Y$ and the forecasted values given by the model $\mathbf{\hat Y}=f(\mathbf B,\mathbf Z)$ using a set of covariates $\mathbf Z$, including lags of $\mathbf Y$. { $\mathbf B$  is the coefficient matrix, which is updated iteratively. The estimate for $\mathbf B$ at iteration $k$, represented by $\mathbf B^k$, is given by}
\begin{equation}
\mathbf B^{k+1}=\mathbf  B^k-(\nabla^2E(\mathbf B^k))^{-1}\nabla E(\mathbf B^k) \ ,
\end{equation}
where $\nabla E$ and $\nabla^2 E$ are the gradient and Hessian of $E$, respectively. {With certain properties, convergence to a certain global minima can be guaranteed~\citep{nocedal2006numerical}: (i)~$\nabla E(\mathbf B)$ is continuously differentiable, and (ii)~$\nabla^2 E(\mathbf B) \nabla E(\mathbf B)$ is convex.}

In order to enable distributed optimization,  $\nabla E$ and $\nabla^2 E$  are required to be decomposable over multiple data owners, {i.e. these functions can be rewritten as the sum of functions that depend {exclusively} on local data from each data owner.}} {\cite{slavkovic2007secure} proposes} a secure logistic regression approach for data {split} by records and features by using secure {multi-party} computation protocols during the Newton-Raphson method iterations. However, and although distributed computing is feasible, there is no sufficient guarantee of data privacy, because it is an iterative process{; even} though an iteration cannot reveal private information, sufficient iterations can: in a logistic regression with data split by features, for each iteration $ k $ the data owners exchange the matrix $\mathbf Z_{A_i} \mathbf B_{A_i}^k$, making it possible to recover the local data $\mathbf Z_{A_i}$ at the end of some iterations~\citep{fienberg2009valid}.

An example of an earlier promising work -- combining logistic regression with the Newton-Raphson method for data distributed by records -- was the Grid binary LOgistic REgression (GLORE) framework~\citep{wu2012g}. The GLORE model is based on model sharing instead of patient-level data, which has motivated {subsequent} improvements, some of which continue to suffer from confidentiality breaches on intermediate results and
other ones resorting to protocols for matrix addition and multiplication. Later, \cite{li2015vertical} explored the issue concerning Newton-Raphson over data distributed by features, considering the existence of a server -- that receives the transformed data and computes the intermediate results, returning them to each data owner. In order to avoid {the} disclosure of local data while obtaining an accurate global solution, the authors apply the kernel trick to obtain the global linear matrix, computed using dot products of local records ($\mathbf Z_{A_i}\mathbf Z_{A_i}\tran$), which can be used to solve the dual problem for logistic regression. However, the authors identified a technical challenge in scaling up the model when the sample size is large, since each record requires a parameter.

\subsubsection{Gradient-Descent Methods}

Different gradient-descent methods have also been explored, aiming to minimize a forecast error function $E$, between the true values $\mathbf Y$ and the forecasted values given by the model $\mathbf{\hat Y}=f(\mathbf B,\mathbf Z)$ using a set of covariates $\mathbf Z$, including lags of $\mathbf Y$. The coefficient matrix $\mathbf B$ is updated iteratively such that {the estimate at} iteration $k$, {$\mathbf B^k$, is given by}
\begin{equation}
\mathbf B^k = \mathbf B^{k-1}+\eta \nabla E(\mathbf B^{k-1})\ ,
\end{equation}
where $\eta$ is the learning rate; it allows the parallel computation {when} the optimization function $E$ is decomposable. A common error function is the {multivariate} least squared error, 
\begin{equation}
E(\mathbf B)=\frac{1}{2}\|\mathbf Y-f(\mathbf B, \mathbf Z)\|^2.
\end{equation}
With certain properties, convergence to a certain global minima can be guaranteed~\citep{nesterov1998introductory}: (i)~$E$ is convex, (ii)~$\nabla E$ is Lipschitz continuous with constant $L$, i.e. for any $\mathbf F$, $\mathbf G$, 
\begin{equation}
\|\nabla E(\mathbf F)-\nabla E(\mathbf G)\|^2\leq L\|\mathbf F-\mathbf G\|^2\ ,
\end{equation}
and (iii)~$\eta\leq 1/L$.

\cite{han2010privacy} proposed a privacy-preserving linear regression for data distributed over features (with shared $\mathbf Y$) by combining distributed gradient-descent with secure protocols, based on pre- or post-multiplication of the data by random private matrices. In the work of~\cite{song2013stochastic}, differential privacy is introduced by adding random noise $\mathbf W$ in the $\mathbf B$ updates, {}
\begin{equation}
\mathbf B^k = \mathbf B^{k-1}+\eta \Big(\nabla E(\mathbf B^{k-1})+ \mathbf W\Big){.}
\end{equation}
{When this iterative process uses a few samples (or even a single sample) randomly selected, rather than the entire data, the process is known as \textit{stochastic} gradient descent (SGD). The authors argue that the trade-off between performance and privacy is most pronounced when smaller batches are used.}


\section{Collaborative Forecasting with VAR: Privacy Analysis}\label{sec:background}

This section presents a privacy analysis focused on the VAR model, {a model for the analysis of multivariate time series} and collaborative forecasting. It is not only used for forecasting tasks in different domains (and with significant improvements over univariate autoregressive models), but also for structural inference where the {main} objective is exploring certain assumptions about the causal structure of the data~\citep{toda1993vector}. A variant with the Least Absolute Shrinkage and Selection Operator (LASSO) regularization is also covered. The critical evaluation of the methods described in {S}ection~\ref{sec:privacy_methods} is conducted from a mathematical and numerical point of view in {S}ection~\ref{sec:privacy_analysis_var}. The solar energy time series dataset and R scripts are published in the online appendages to this paper (see~\ref{app:data_code}).

\subsection{VAR Model Formulation} \label{sec:admmVARLASSO}

Let $\{\mathbf{y}_t\}_{t=1}^T$ be an $n$-dimensional multivariate time series, where $n$ is the number of data owners. Then, $\{\mathbf{y}_t\}_{t=1}^T$ follows a VAR model with $p$ lags, {represented} as $\text{VAR}_{n}(p)$,  {when} the following relationship holds
\begin{equation} \label{eq:VAR}
\mathbf y_t = \boldsymbol \eta + \sum_{\ell=1}^p  \mathbf y_{t-\ell}\mathbf B^{(\ell)}  + \boldsymbol \varepsilon_t \ ,
\end{equation}
for $t = 1, \dots, T$, 
where $\boldsymbol\eta = [\eta_{1},\dots,\eta_{n}]$ is the constant intercept (row) vector, $\boldsymbol\eta \in \mathbb{R}^{n}$; $\mathbf{B}^{(\ell)}$ represents the coefficient matrix at lag $\ell=1,...,p$, $\mathbf{B}^{(\ell)} \in \mathbb{R}^{n\times n}$, and the coefficient associated with lag $\ell$ of time series $i$ (to estimate time series $j$) is {positioned at} $(i,j)$ of $\mathbf{B}^{(\ell)}$, for $i,j=1,...,n$;
 and $\boldsymbol\varepsilon_t = [\varepsilon_{1,t},\allowbreak\dots,\allowbreak\varepsilon_{n,t}]$, $\boldsymbol\varepsilon_t \in \mathbb{R}^{n}$, indicates a white noise vector that is independent and identically distributed with mean zero and nonsingular covariance matrix. By simplification, $\mathbf y_t$ is assumed to follow a centered process, $\boldsymbol{\eta} =\mathbf{0}$, i.e., as a vector of zeros of appropriate dimension. A compact representation of a $\text{VAR}_n(p)$ model reads as
 \begin{equation}
     \mathbf Y=\mathbf{ZB+ E} \ ,
 \end{equation}
 where 
\begin{equation*}
\mathbf Y = \left[ 
\begin{tabular}{l}
$\mathbf y_1$\\ 
\vdots\\
$\mathbf y_T$
\end{tabular}\right], \quad
\mathbf B = \left[ 
\begin{tabular}{l}
$\mathbf B^{(1)}$\\ 
\vdots\\
$\mathbf B^{(p)}$
\end{tabular}\right], \quad
\mathbf Z = \left[ 
\begin{tabular}{l}
$\mathbf z_1$\\ 
\vdots\\
$\mathbf z_T$
\end{tabular}\right], \text{ and } \, 
\mathbf E = \left[ 
\begin{tabular}{l}
$\boldsymbol{\varepsilon}_1$\\ 
\vdots\\
$\boldsymbol{\varepsilon}_T$
\end{tabular}\right]{,}
\end{equation*} 
are obtained by joining the vectors row-wise, and {defining}, respectively, the $T\times n$ response matrix, the $np\times n$ coefficient matrix, the $T\times np$ covariate matrix and the $T\times n$ error matrix, with $\mathbf z_t=[\mathbf{y}_{t-1},\allowbreak\dots,\allowbreak\mathbf{y}_{t-p}]$.

Notice that the VAR formulation adopted in this paper is not the usual $\mathbf{Y}\tran=\mathbf{B}\tran \mathbf{Z}\tran+ \mathbf{E}\tran$ because a large proportion of the literature in privacy-preserving techniques {derives} from the standard linear regression problem, in which each row is a {record} and each column is a feature.

Notwithstanding the high potential of the VAR model for collaborative forecasting, {namely} by linearly combining time series from the different data owners, data privacy or confidentiality issues might hinder this approach. For instance, renewable energy companies, competing in the same electricity market, will never share their electrical energy production data, even if this leads to a forecast error improvement in all individual forecasts.

For classical linear regression models, there are several techniques for estimating coefficients without sharing private information. However, in the VAR model, the data is divided by features (Figure~\ref{fig:datasplit})  and the variables to be forecasted are also covariates, which is challenging for privacy-preserving techniques ({especially} because it is also necessary to protect the data matrix $\mathbf Y$, as illustrated in Figure~\ref{fig:XZmatrices}). In the remaining of the paper, $\mathbf{Y}_{A_i}\in \mathbb{R}^{T \times 1}$ and $\mathbf{Z}_{A_i}\in \mathbb{R}^{T \times p}$  {represent} the target and covariate matrix for the $i$-th data owner, respectively, when defining a VAR model.  Therefore, the covariates and target matrices are obtained by joining the individual matrices column-wise, i.e. $\mathbf Z=[\mathbf Z_{A_1},\dots, \mathbf Z_{A_n}]$ and $\mathbf Y=[\mathbf Y_{A_1},\dots, \mathbf Y_{A_n}]$. For distributed computation, the coefficient matrix of data owner $i$ is denoted by $\mathbf B_{A_i}\in \mathbb{R}^{p\times n}, i=1,\dots,n$.

\begin{figure}
\centering
\includegraphics[width=\linewidth,trim={0 0 0 0cm},clip]{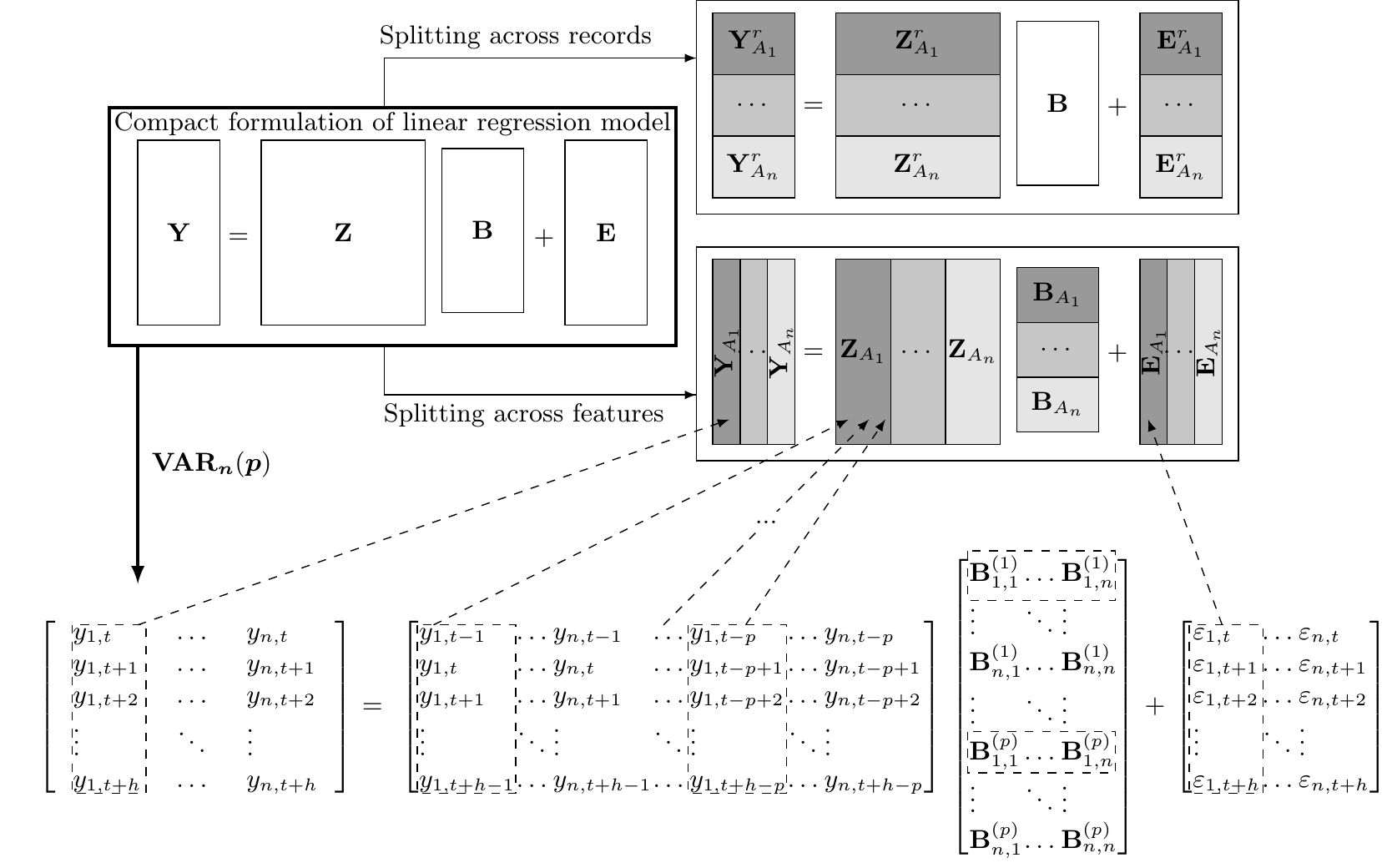}
\caption{Common data division structures and VAR model.}
\label{fig:datasplit}
\end{figure}

\begin{figure}
        \centering
\setlength{\tabcolsep}{3pt} 
        \begin{tabular}{c|cccccc}
            \multicolumn{1}{c}{$\overbrace{\rule{1.5cm}{0pt}}^{\text{Y of i-th data owner}}$}  &  \multicolumn{6}{c}{$\overbrace{\rule{6.4cm}{0pt}}^{\text{covariates values of i-th data owner}}$} \\
             $y_{i,t}$  &  $y_{i,t-1}$ &  $y_{i,t-2}$ &  $y_{t-3,i}$ &  $\dots$ &  $y_{i,t-p+1}$ &  $y_{i,t-p}$ \\
             $y_{i,t+1}$  & \color{gray} $y_{i,t}$ & \color{gray}$y_{i,t-1}$ & \color{gray} $y_{t-2,i}$ & \color{gray} $\dots$ & \color{gray} $y_{i,t-p+2}$ & \color{gray} $y_{i,t-p+1}$\\
             $y_{i,t+2}$  & \color{gray} $y_{i,t+1}$ & \color{gray} $y_{i,t}$ & \color{gray} $y_{t-1,i}$ & \color{gray} $\dots$ & \color{gray} $y_{i,t-p+3}$ & \color{gray} $y_{i,t-p+2}$\\
             \vdots &  \vdots & \vdots & \vdots & \vdots & \vdots & \vdots \\
             $y_{i,t+h}$  & \color{gray} $y_{i,t+h-1}$ & \color{gray} $y_{i,t+h-2}$ & \color{gray} $y_{t+h-3,i}$& \color{gray} $\dots$ & \color{gray} $y_{i,t+h-p+1}$ & \color{gray} $y_{i,t+h-p}$\\
        \end{tabular}
\caption{Illustration of the data used by the $i$-th data owner when fitting a VAR model.}
\label{fig:XZmatrices}
\end{figure}

\subsection{Estimation in VAR Models} \label{sec:var_estimation}

Commonly, when the number of covariates included, $np$, is substantially smaller than the length of the time series, $T$, the VAR model can be fitted using multivariate least squares {solution given by} 
\begin{equation} \label{eq:olsVAR}
	\hat{\mathbf B}_\text{LS}=\argmin_{\mathbf B} \left(\|\mathbf Y-\mathbf{ZB}\|_2^2\right) \ ,
\end{equation}
 where $\|.\|_r$ represents both vector and matrix $L_r$ norms.  However, in collaborative forecasting, as the number of data owners increases, as well as the number of lags, it becomes crucial to use regularization techniques, such as LASSO, in order to introduce sparsity into the coefficient matrix estimated by the model. 
In the standard LASSO-VAR approach (see \cite{Nicholson2017} for different variants of the LASSO regularization in the VAR model), the coefficients are {given by}
\begin{equation} \label{eq:lasso}
	\hat{\mathbf B}=\argmin_{\mathbf B} \left(\frac{1}{2}\|\mathbf Y-\mathbf{ZB}\|_2^2+\lambda \|\mathbf B\|_1\right),
\end{equation}
where $\lambda > 0$ is a scalar penalty parameter.

With the addition of the LASSO regularization term, the {convex} objective function in~(\ref{eq:lasso}) becomes non-differentiable, thus limiting the variety of optimization techniques that can be employed. In this domain, ADMM {(which was described in}~\ref{sec:admm}) is a {widespread} and computationally efficient technique that enables the parallel estimation for data divided by features.
The ADMM formulation of the non-differentiable cost function associated to LASSO-VAR model in \eqref{eq:lasso} solves the optimization problem
\begin{equation} \label{eq:admmlasso}
	\min_{\mathbf B, \mathbf H} \Big(\frac{1}{2}\|\mathbf Y-\mathbf{ZB}\|_2^2+\lambda \|\mathbf H\|_1\Big) \text{ subject to } \mathbf{H}=\mathbf{B} \ ,
\end{equation}
which differs from~\eqref{eq:lasso} by splitting $\mathbf B$ into two parts ($\mathbf B$ and $\mathbf H$).
This allows splitting the objective function in two distinct objective functions, $f(\mathbf B)=\frac{1}{2}\|\mathbf Y-\mathbf{ZB}\|_2^2$ and $g(\mathbf H)=\lambda \|\mathbf H\|_1$. The augmented Lagrangian{~\citep{boyd2011distributed}} of this problem is 
\begin{equation} \label{eq:augLagrangean}
	L_\rho(\mathbf B, \mathbf H, \mathbf W)=\frac{1}{2}\|\mathbf Y-\mathbf{ZB}\|_2^2+\lambda \|\mathbf H\|_1+\mathbf W\tran(\mathbf B-\mathbf H)+\frac{\rho}{2}\|\mathbf B - \mathbf H\|_2^2 \ ,
\end{equation}
where $\mathbf W$ is the dual variable and $\rho>0$ is the penalty parameter. The scaled form of this Lagrangian is 
\begin{equation} \label{eq:augLagrangeanX}
	L_\rho(\mathbf B, \mathbf H, \mathbf U)=\frac{1}{2}\|\mathbf Y-\mathbf{ZB}\|_2^2+\lambda \|\mathbf H\|_1+\frac{\rho}{2}\|\mathbf B-\mathbf H+\mathbf U\|^2-\frac{\rho}{2}\|\mathbf U\|^2 \ ,
\end{equation}
where $\mathbf U=(1/\rho)\mathbf W$ is the scaled dual variable associated with the constrain $\mathbf B=\mathbf H$. Hence, {according to~\eqref{ADMM},} the ADMM formulation for LASSO-VAR consists in the following iterations~\citep{Cavalcante2017},
\begin{equation}  \label{ADMMLASSOVAR} \left
		\{
            \begin{array}{l} \displaystyle 
              \mathbf B^{k+1} :=\argmin_{\mathbf B}\Big( \frac{1}{2}\|\mathbf{Y -ZB}\|_2^2+\frac{\rho}{2}\|\mathbf{B}- \mathbf H^k+\mathbf U^k\|_2^2\Big) \\
              \displaystyle 
              \mathbf H^{k+1}:=\argmin_{\mathbf H}\Big(\lambda \|\mathbf H\|_1 + \frac{\rho}{2}\|\mathbf B^{k+1}-\mathbf H+\mathbf U^k\|_2^2\Big)\\
              \mathbf U^{k+1}:=\mathbf U^k+\mathbf B^{k+1}-\mathbf H^{k+1}.
            \end{array}
          \right.
\end{equation}

{Concerning the LASSO-VAR model, and since data is naturally divided by features (i.e. $\mathbf Y=\allowbreak[\mathbf Y_{A_1},\allowbreak\dots,\allowbreak \mathbf Y_{A_n}]$, $\mathbf Z=\allowbreak[\mathbf Z_{A_1},\allowbreak \dots, \allowbreak \mathbf Z_{A_n}]$ and $\mathbf B=[\mathbf B\tran_{A_1},\dots, \mathbf B\tran_{A_n}]\tran$) and the functions $\|\mathbf{Y-ZB}\|_2^2$ and $\|\mathbf B\|_1$ are decomposable (i.e. $\|\mathbf{Y-ZB}\|_2^2=\|\mathbf Y -\sum_{i=1}^n\mathbf Z_{A_i}\mathbf B_{A_i}\|_ 2^2$ and $\|\mathbf B\|_1=\sum_{i=1}^n \mathbf \|\mathbf B_{A_i}\|_1$), the model fitting problem~\eqref{eq:lasso} becomes 
\begin{equation} \label{eq:lasso_split}
	\argmin_{\boldsymbol\Gamma} \left(\frac{1}{2}\|\mathbf Y-\sum_{i=1}^n\mathbf{Z}_{A_i}\mathbf{B}_{A_i}\|_2^2+\lambda \sum_{i=1}^n\|\mathbf B_{A_i}\|_1\right),
\end{equation}
$\boldsymbol\Gamma=\{\mathbf B_{A_1},\dots, \mathbf B_{A_n}\}$, which is rewritten as
\begin{equation} 
	\argmin_{\boldsymbol\Gamma'} \left(\frac{1}{2}\|\mathbf Y-\sum_{i=1}^n\mathbf{H}_{A_i}\|_2^2+\lambda \sum_{i=1}^n\|\mathbf B_{A_i}\|_1\right) \text{ s.t. } {\mathbf{B}_{A_1}\mathbf{Z}_{A_1}=\mathbf{H}_{A_1}, \ \dots, \ \mathbf{B}_{A_n}\mathbf{Z}_{A_n}=\mathbf{H}_{A_n} \ ,} 
\end{equation}
{$\boldsymbol \Gamma' =\{\mathbf B_{A_1}, \dots, \mathbf B_{A_n}, \mathbf H_{A_1}, \dots, \mathbf H_{A_n}\}$}, while the corresponding distributed ADMM formulation~\citep{boyd2011distributed,Cavalcante2017}} is the one presented in the system of equations~(\ref{eq:admmcolumn}),
\begin{subequations}\label{eq:admmcolumn}
\begin{equation}\label{eq:admmcolumnA}
  \mathbf B_{A_i}^{k+1} =\argmin_{\mathbf B_{A_i}}\left( \frac{\rho}{2}\|\mathbf{Z}_{A_i}\mathbf{B}_{A_i}^k +\overline{\mathbf H}^k-\overline{\mathbf{ZB}}^{k}-\mathbf U^k-\mathbf Z_{A_i}\mathbf B_{A_i}  \|_2^2 +\lambda\|\mathbf B_{A_i}\|_1\right),
\end{equation}    
\begin{equation}\label{eq:admmcolumnB}
 \overline{\mathbf H}^{k+1} =\frac{1}{N+\rho}\left(\mathbf Y+ \rho \overline{\mathbf{ZB}}^{k+1}+\rho \mathbf U^k\right),
\end{equation}
\begin{equation}\label{eq:admmcolumnC}
\mathbf U^{k+1} =\mathbf U^k+ \overline{\mathbf{ZB}}^{k+1}-\overline{\mathbf H}^{k+1},
\end{equation}
\end{subequations}
\normalsize
where $\overline{\mathbf{ZB}}^{k+1} = \frac{1}{n}\sum_{j=1}^n \mathbf{Z}_{A_j}\mathbf{B}^{k+1}_{A_j}$ and  $\mathbf{B}_{A_i}^{k+1} \in \mathbb{R}^{p \times n}$, $\mathbf{Z}_{A_i}\in \mathbb{R}^{T \times p}$,$\mathbf{Y}\in \mathbb{R}^{T \times n}$, $\overline{\mathbf{H}}^k, \mathbf{U} \in \mathbb{R}^{T\times n}$, $i=1,..., n$. 

{Although the parallel computation is an appealing property for the design of a privacy-preserving approach, ADMM is an iterative optimization process that requires intermediate calculations, thus a careful analysis is necessary to evaluate if some confidentiality breaches can occur at the end of some iterations.}

\subsection{Privacy Analysis} \label{sec:privacy_analysis_var}

\subsubsection{Data Transformation with Noise Addition}\label{subsec:noiseaddition_experiments}

{This section presents experiments with simulated data and solar energy data collected from a smart grid pilot in Portugal. The objective is to quantify the impact of data distortion (through noise addition) into the model forecasting skill.}
 
\textit{a) Synthetic Data:} 
An experiment has been performed to add random noise from the Gaussian {distribution with zero mean and variance $b^2$, Laplace distribution with zero mean and scale parameter $b$ and Uniform distribution with support $[-b,b]$, represented by  $\mathcal N(0,b^2)$, $\mathcal L(0,b)$ and $\mathcal U(-b,b)$, respectively.}
Synthetic data generated by VAR processes are used to measure the differences between the coefficients' values when adding noise to the data. The simplest case considers a VAR with two data owners and two lags described by
$$\left( \begin{matrix}
	y_{1,t}\ y_{2,t}\\
\end{matrix} \right){=}
\left( \begin{matrix}
	y_{1,t-1} \ y_{2,t-1}\ y_{1,t-2} \ y_{2,t-2}\\
\end{matrix} \right)
\left( \begin{matrix}
	0.5 & 0.3 \\ 
	0.3 & 0.75 \\
	-0.3 & -0.05\\
	-0.1 & -0.4\\
\end{matrix} \right){+}
\left( \begin{matrix}
	\varepsilon_{1,t} \ \varepsilon_{2,t}\\
\end{matrix} \right).$$
The {second} case {includes} ten data owners and three lags, introducing a high percentage of null coefficients ($\approx 86\%$){, with} Figure~\ref{fig:var10_3} illustrating the considered coefficients. Since a specific configuration can generate {various} distinct trajectories, 100 simulations are performed for each specified VAR model, each of them with 20,000 {timestamps}. For both simulated datasets, the errors $\boldsymbol\varepsilon_t$ were assumed to follow a multivariate Normal distribution with a zero mean vector and a covariance matrix equal to the identity matrix of appropriate dimensions. Distributed ADMM (detailed in {S}ection~\ref{sec:admm}) was used to estimate the LASSO-VAR coefficients, considering two different noise characterizations, $b \in \{0.2,0.6\}$.

\begin{figure}
\centering
\includegraphics[width=0.9\linewidth]{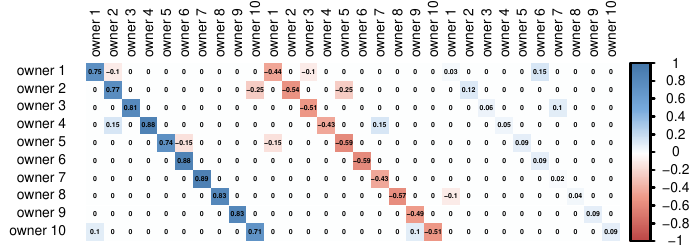}
\caption{Transpose of the coefficient matrix used to generate the VAR with 10 data owners and 3 lags.}
\label{fig:var10_3}
\end{figure}

Figure~\ref{fig:coefs_VARprocesses} summarizes the mean and the standard deviation of the absolute difference {between the real and estimated coefficients,} for both VAR processes, from 100 simulations. The greater the noise $b$, the greater the distortion of the estimated coefficients. Moreover, the Laplace distribution, which has desirable properties to make the data private according to the differential privacy framework, registers the greater distortion in the estimated model.

\begin{figure}
\centering
\includegraphics[width=0.9\linewidth,trim={0 0 0 0.6cm},clip]{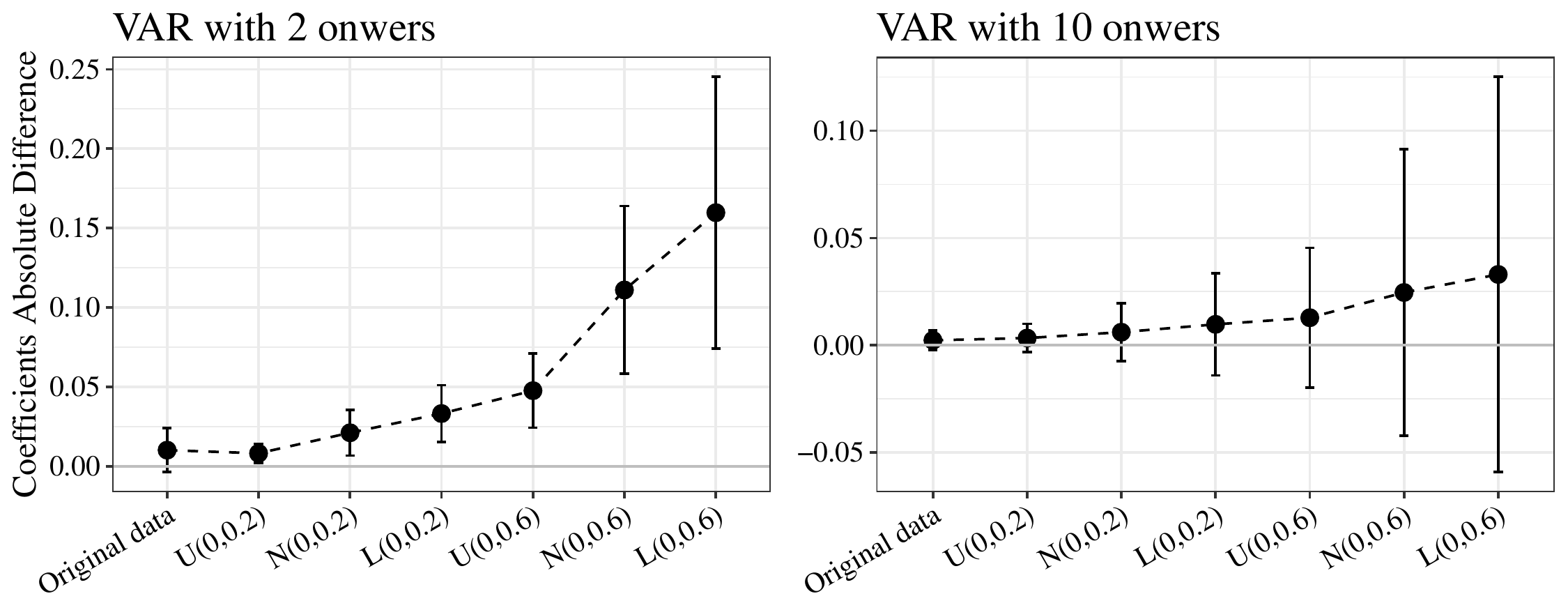}
\caption{Mean {$\pm$} standard deviation {for} the absolute difference {between the real and estimated coefficients} (left: VAR with 2 data owners, right: VAR with 10 data owners).}
\label{fig:coefs_VARprocesses}
\end{figure}

Using the original data, the ADMM solution tends to stabilize after 50 iterations, and the value of the coefficients is correctly estimated (the difference is approximately zero). {Regarding} the distorted time series, it converges faster, but the coefficients deviate from the real ones. {In fact}, adding noise contributes to decreasing the absolute value of the coefficients, i.e.\ the relationships between the time series {become} weakened.

These experiments allow drawing some conclusions about the use of differential privacy. The Laplace distribution has {advantageous} properties, since it ensures $\varepsilon$-differential privacy {when} random noise follows $\mathcal L(0,\frac{\Delta f_1}{\varepsilon})$. For the VAR with two data owners, $\Delta f_1 \approx 12$ since the observed values are in the interval $[-6,6]$. Therefore, $\varepsilon=20$ when $\mathcal L(0, 0.6)$ and $\varepsilon=15$ when $\mathcal L(0, 0.8)$, meaning that the data still encompass much relevant information. Finally, to verify the impact of noise addition into forecasting performance, Figure~\ref{fig:imp_VARprocesses} illustrates the improvement of each estimated VAR$_2(2)$ model (with and without noise addition) over the Autoregressive (AR) model estimated with original time series, in which collaboration is not used. This improvement is measured in terms of Mean Absolute Error (MAE) and Root Mean Squared Error (RMSE) values. {Concerning} the case with ten data owners, and when using data without noise, seven data owners improve their forecasting performance, which was expected from the coefficient matrix in Figure~\ref{fig:var10_3}. When Laplacian noise is applied to the data, only one data owner (the first one) improves its forecasting skill (when compared to the AR model) by using the estimated VAR model. Even though the masked data continues to provide relevant information, the model obtained for the Laplacian noise performs worse than the AR model for the second data owner, making the VAR useless for the majority of the data owners. 

\begin{figure}
\centering
\includegraphics[width=\linewidth,trim={0 0 0 0cm},clip]{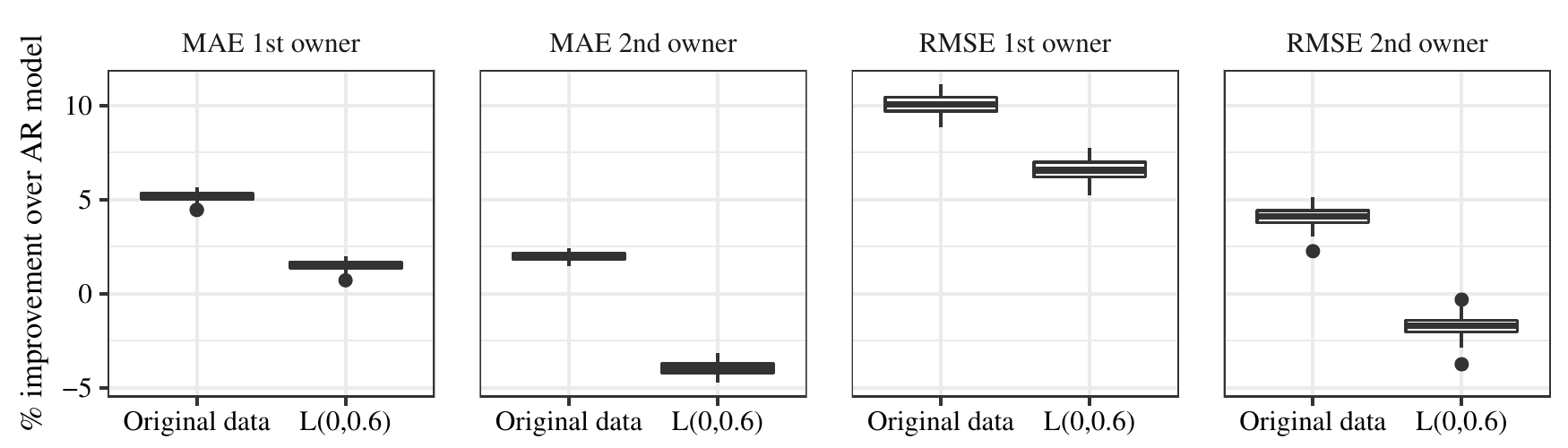}
\caption{Improvement (\%) of VAR$_2(2)$ model over AR(2) model, in terms of MAE and RMSE for synthetic data.}
\label{fig:imp_VARprocesses}

\begin{subfigure}{.5\textwidth}
  \centering
  \includegraphics[width=\linewidth]{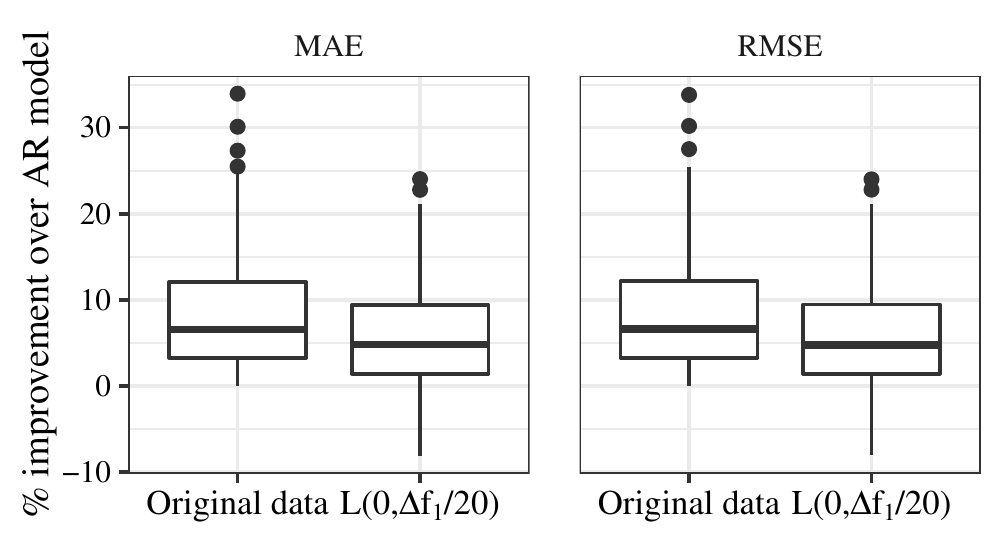}
  \caption{VAR$_{2}(2)$ model.}
\end{subfigure}%
\begin{subfigure}{.5\textwidth}
  \centering
  \includegraphics[width=\linewidth]{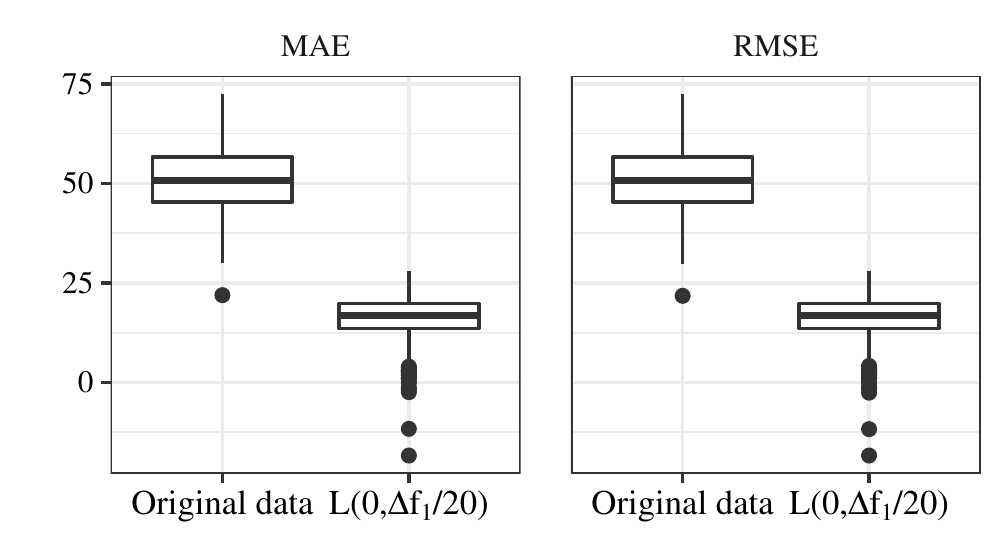}
  \caption{VAR$_{10}(3)$ model.}
\end{subfigure}%
\caption{Improvement (\%) of VAR model over AR model, in terms of MAE and RMSE for synthetic data.}%
\label{fig:imp_genVARprocesses}%
\end{figure}

However, the results cannot be generalized for all VAR models, especially regarding the illustrated VAR$_{10}(3)$, which is very close to the AR(3) model. Given that, a third experiment is proposed, in which 200 random coefficient matrices are generated for a stationary VAR$_{2}(2)$ and VAR$_{10}(3)$ following the algorithm proposed by~\cite{ansley1986note}. Usually, the generated coefficient matrix has no null entries and the higher values are not necessarily found on diagonals. Figure~\ref{fig:imp_genVARprocesses} illustrates the improvement for each data owner when using a VAR model (with and without noise addition) over the AR model. In this case, the percentage of times the AR model performs better than the VAR model with distorted data is smaller, but the degradation of the models is still noticeable, especially in relation to the case with 10 data owners.

\textit{b) Real Data:}
\label{realdata}
The real dataset encompasses hourly time series of solar power generation from 44 micro-generation units located in \'Evora city (Portugal){, covering} the period from February 1, 2011 to March 6, 2013. As in~\citet{cavalcante2017solar}, records corresponding to a solar zenith angle higher than 90$^\circ$ were removed, in order to take off nighttime hours (i.e., hours without any generation). Furthermore, {and} to make the time series stationary, a normalization of the solar power was applied by using a clear-sky model (see~\cite{Bacher2009}) that gives an estimate of the solar power in clear sky conditions at any given time. The power generation for the next hour is modeled {through} the VAR model, which combines data from 44 data owners {and} considers 3 non-consecutive lags (1, 2 and 24h).  
Figure~\ref{fig:realCS} (a) summarizes the improvement for the 44 PV power plants over the autoregressive model, in terms of MAE and RMSE. The quartile 25\% allows concluding that MAE improves at least 10\% for 33 of the 44 PV power plants, {when} data owners share their observed data. {As to} RMSE the improvement is not so {significant,} but is still greater than zero. Although the data obtained after Laplacian noise addition keeps its temporal dependency, as illustrated in Figure~\ref{fig:realCS} (b), the corresponding VAR model is useless for 4 of the 44 data owners. When considering RMSE, 2 of the 44 data owners obtain better results by using an autoregressive model. Once again, the resulting model suffers a significant reduction in  terms of forecasting capability.

\begin{figure}
\centering
\begin{subfigure}{.5\textwidth}
  \centering
  \includegraphics[width=\linewidth]{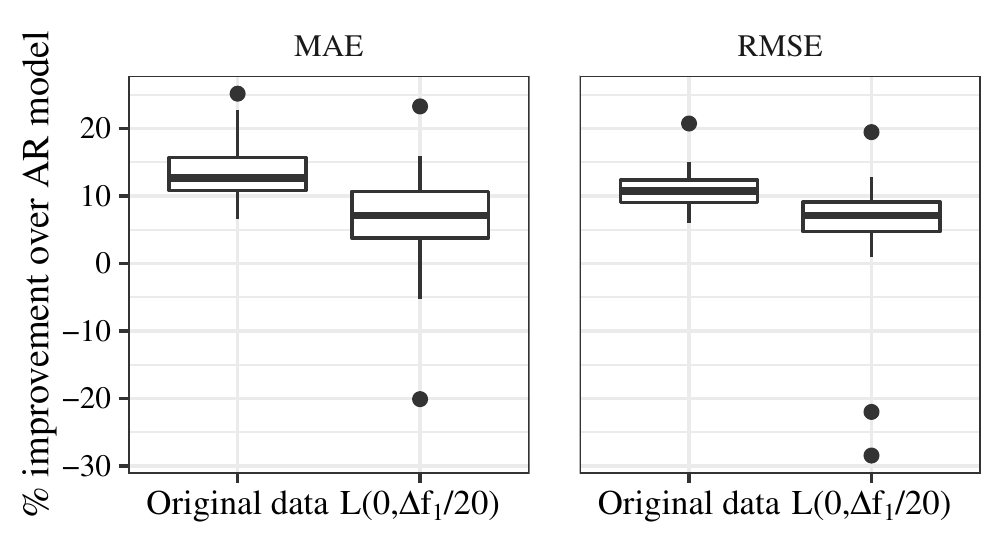}
  \caption{Improvement (\%) of VAR$_{44}$ model over AR model, in terms of MAE and RMSE.}
  \label{fig:sub1}
\end{subfigure}%
\begin{subfigure}{.5\textwidth}
  \centering
  \includegraphics[width=\linewidth]{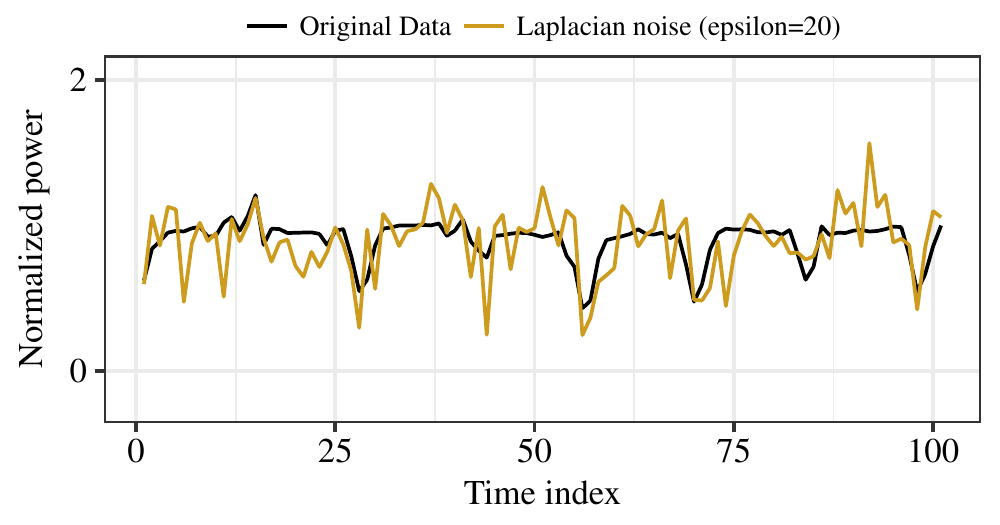}
  \caption{Example of the normalized time series.}
  \label{fig:sub2}
\end{subfigure}%
\caption{Results for real case-study with solar power time series.}%
\label{fig:realCS}%
\end{figure}

\subsubsection{Linear Algebra-based Protocols}\label{subsec:emp_analysis}


Let us consider the case with two data owners. Since the multivariate least squares estimate for the VAR model with covariates $\mathbf Z=[\mathbf{Z}_{A_1},\mathbf{Z}_{A_2}]$ and target $\mathbf Y=[\mathbf{Y}_{A_1},\mathbf{Y}_{A_2}]$ 
is 
\begin{align}
\hat{\mathbf B}_\text{LS}&= \left(\left[
\begin{tabular}{c}
$\mathbf{Z}\tran_{A_1}$\\ 
$\mathbf{Z}\tran_{A_2}$\\
\end{tabular}\right][\mathbf{Z}_{A_1},\mathbf{Z}_{A_2}]\right)^{-1}
\left(\left[
\begin{tabular}{c}
$\mathbf{Z}\tran_{A_1}$\\ 
$\mathbf{Z}\tran_{A_2}$\\
\end{tabular}\right][\mathbf{Y}_{A_1},\mathbf{Y}_{A_2}]\right)\\
&=
\left(
\begin{tabular}{cc}
$\mathbf{Z}\tran_{A_1}\mathbf{Z}_{A_1}$ & $\mathbf{Z}\tran_{A_1}\mathbf{Z}_{A_2}$\\ 
$\mathbf{Z}\tran_{A_2}\mathbf{Z}_{A_1}$ & $\mathbf{Z}\tran_{A_2}\mathbf{Z}_{A_2}$\\
\end{tabular}\right)^{-1}
\left(
\begin{tabular}{cc}
$\mathbf{Z}\tran_{A_1}\mathbf{Y}_{A_1}$ & $\mathbf{Z}\tran_{A_1}\mathbf{Y}_{A_2}$\\ 
$\mathbf{Z}\tran_{A_2}\mathbf{Y}_{A_1}$ & $\mathbf{Z}\tran_{A_2}\mathbf{Y}_{A_2}$\\
\end{tabular}\right),
\end{align}
the data owners need to jointly compute $\mathbf Z_{A_1}\tran \mathbf Z_{A_2}$, $\mathbf Z_{A_1}\tran \mathbf Y_{A_2}$ and $\mathbf Z_{A_2}\tran \mathbf Y_{A_1}$. 

As mentioned in the introduction of {S}ection~\ref{linear-algebra-protocols}, the work of~\cite{du2004privacy} proposes protocols for secure matrix multiplication for the situations where two data owners observe the same common target {matrix} and different confidential covariates. Unfortunately, {without assuming a trusted third entity for generating random matrices,} the proposed protocol fails when applied to the VAR model {because $2(T-1)p$ values} of the covariate matrix $\mathbf Z\in \mathbb{R}^{T\times 2p}$ are {included in} the target {matrix} $\mathbf Y\in \mathbb{R}^{T\times 2}$, which {is} also undisclosed.  {Additionally, $\mathbf Z_{A_i}\in \mathbb{R}^{T\times p}$ has $T+p-1$ unique values instead of $Tp$ -- see Figure~\ref{fig:XZmatrices}.}

\begin{proposition}
{Consider the case in which two data owners, with private data $\mathbf Z_{A_i}\in \mathbb{R}^{T\times p}$ and $\mathbf Y_{A_i}\in \mathbb{R}^{T\times 1}$, want to estimate a VAR model without trusting a third entity, $i=1,2$. Assume that the $T$ records are consecutive, as well as the $p$ lags.
The multivariate least squares estimate for the VAR model with covariates $\mathbf Z=[\mathbf{Z}_{A_1},\mathbf{Z}_{A_2}]$ and target $\mathbf Y=[\mathbf{Y}_{A_1},\mathbf{Y}_{A_2}]$ requires the computation of $\mathbf Z_{A_1}\tran \mathbf Z_{A_2}$, $\mathbf Z_{A_1}\tran \mathbf Y_{A_2}$ and $\mathbf Z_{A_2}\tran \mathbf Y_{A_1}$. 

If data owners use the protocol proposed by~\cite{du2004privacy} for computing such matrices, then the information exchanged allows to recover data matrices.
}
\end{proposition}

\begin{proof}

{As in \cite{du2004privacy}, l}et us consider the case with two data owners without a third entity generating random matrices. 

In order to compute $\mathbf Z_{A_1}\tran \mathbf Z_{A_2}$ both data owners define a matrix $\mathbf M \in \mathbb R^{T\times T}$ and compute its inverse $\mathbf M^{-1}$. {Then, the protocol defines that
\begin{align*}
\mathbf Z_{A_1}\tran \mathbf Z_{A_2} &= \mathbf Z_{A_1}\tran \mathbf{MM^{-1}}\mathbf Z_{A_2} = \mathbf A[\mathbf M_\text{left}, \mathbf M_\text{right}]
\left[\begin{tabular}{l}
$(\mathbf M^{-1})_\text{top}$ \\
$(\mathbf M^{-1})_\text{bottom}$
\end{tabular}\right]
\mathbf Z_{A_2} \\
& = \underbrace{\mathbf Z_{A_1}\tran \mathbf{M}_\text{left}(\mathbf M^{-1})_\text{top}\mathbf Z_{A_2}}_\text{derived by Owner~\#1} + \underbrace{\mathbf Z_{A_1}\tran \mathbf{M}_\text{right}(\mathbf M^{-1})_\text{bottom}\mathbf Z_{A_2}}_\text{derived by Owner~\#2} \ ,
\end{align*}
requiring data owners to share $\mathbf Z\tran_{A_1}\mathbf M_\text{right} \in \mathbb{R}^{p\times T/2}$ and $(\mathbf M^{-1})_\text{top}\mathbf Z_{A_2}\in \mathbb{R}^{T/2 \times p}$, respectively.} This implies that each data owner shares $p T/2$ values.

Similarly, the computation of $\mathbf Z_{A_1}\tran \mathbf Y_{A_2}$ implies that data owners define a matrix $\mathbf M^*$, and share $\mathbf Z\tran_{A_1}\mathbf M^*_\text{right} \in \mathbb{R}^{p\times T/2}$ and $(\mathbf {M^*}^{-1})_\text{top}\mathbf Y_{A_2}\in \mathbb{R}^{T/2 \times p}$, respectively, providing new $p T/2$ values. This means that Owner~\#2 receives $\mathbf Z\tran_{A_1}\mathbf M_\text{right}$ and $\mathbf Z\tran_{A_1}\mathbf M^*_\text{right}$, i.e. $Tp$ values, and {may} recover $\mathbf Z_{A_1}$ which consists of $T p$ values, representing a confidentiality breach. Furthermore, when considering a VAR model with $p$ lags, $\mathbf Z_{A_1}$ has $T+p-1$ unique values, meaning less values to recover. {Analogously}, Owner~\#1 may recover $\mathbf Z_{A_2}$ through the matrices shared for the computation of $\mathbf Z_{A_1}\tran \mathbf Z_{A_2}$ and $\mathbf Z_{A_2}\tran \mathbf Y_{A_1}$.

Lastly, when considering a VAR with $p$ lags, $\mathbf Y_{A_i}$ only has $p$ values that are not in $\mathbf Z_{A_i}$. {Since, while computing $\mathbf Z_{A_1}\tran \mathbf Y_{A_2}$, Owner~\#1 receives $T/2$ values from $(\mathbf M*^{-1})_\text{top}\mathbf Y_{A_2}\in \mathbb{R}^{T/2 \times 1}$, a confidentiality breach can occur (in general $T/2>p$). In the same way, Owner~\#2 recovers $\mathbf Y_{A_1}$ when computing $\mathbf Z_{A_2}\tran \mathbf Y_{A_1}$.}
{\qed}
\end{proof}

The main disadvantage of the linear algebra-based methods is that they do not take into account that, in the VAR model, both target variables and covariates are private and that a large proportion of the {covariates} matrix is determined by knowing the target variables. This means that the data shared between data owners may be enough for competitors to be able to reconstruct the original data. For the method proposed by~\cite{karr2009privacy}, a consequence of such data is that the assumption $rank\left(\mathbf{(I-WW\tran)} \allowbreak \mathbf C\right)=m{-}g$ may {still provide a number of linearly independent equations on the other data owner's data, which is enough for recovering their data.}

\subsubsection{ADMM Method and Central Node}\label{subsec:admmAnalysis}

The work of~\cite{zhang2018distributed} appears to be a promising approach for dealing with the problem of private data during the ADMM iterative process described by~\eqref{eq:admmcolumn}. 
Based on~\cite{zhang2018distributed}, for each iteration $k$, each data owner $i$ communicates their local results, $\mathbf{Z}_{A_i}\mathbf B^{k+1}_{A_i}$, to the central node, $\mathbf Z_{A_i}\in \mathbb{R}^{T\times p},\mathbf B^{k+1}_{A_i}\in \mathbb{R}^{p\times n}, i=1,\dots,n$. Then, the central node computes the intermediate matrices in \eqref{eq:admmcolumnB}-\eqref{eq:admmcolumnC} and returns the matrix $\overline{\mathbf H}^k-\overline{\mathbf{ZB}}^k-\mathbf U^k$ to each data owner, {in order} to update $\mathbf{B}_{A_i}$ in the next iteration, as seen in~\eqref{eq:admmcolumnA}. Figure~\ref{ADMMprivate} illustrates the methodology for the LASSO-VAR with three data owners. In this solution, there is no direct exchange of private data. {However}, as {presented} next, not only can the central node recover the original data, but also individual data owners can obtain a good estimation of the data used by the competitors. 

{
\begin{proposition}
In the most optimistic scenario, without repeated values in $\mathbf{Y}_{A_i}\in\mathbb{R}^{T\times 1}$ and $\mathbf Z_{A_i}\in\mathbb{R}^{T\times p}$, when applying the algorithm from~\cite{zhang2019admm} to solve the LASSO-VAR model in~\eqref{eq:admmcolumn}, the central agent can recover the sensible data at the end of 
\begin{equation}
k=\ceil[\bigg]{\frac{Tp}{Tn-pn}}    
\end{equation} 
iterations{, where $\ceil{x}$ denotes the ceiling function}.
\end{proposition}
}

{
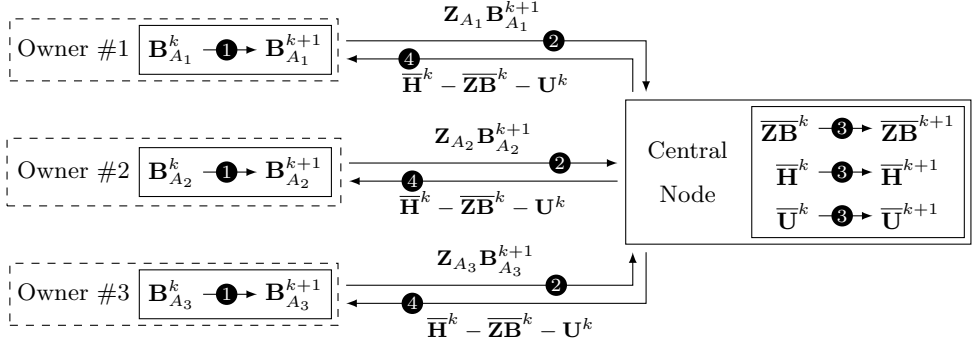
\begin{figure}
\newcommand\passos[3]{\begin{tikzpicture}\node[inner sep=0](a){#1};\node[right=2.8em of a, inner sep=0](b){#2};\draw[-latex](a)--node[pos=0.45, bolanro]{#3}(b);\end{tikzpicture}}
\centering
\begin{tikzpicture}[node distance=0.5cm and 6.5cm, shorten >=1mm, shorten <=1mm, font=\small]  
\node[inner sep=2mm, align=center] (A1) {Central\\Node};
\node[align=center, draw, right=1mm of A1] (A2) {\passos{$\overline{\mathbf{ZB}}^k$}{$\overline{\mathbf{ZB}}^{k+1}$}{3}\\\passos{$\overline{\mathbf{H}}^k$}{$\overline{\mathbf{H}}^{k+1}$}{3}\\\passos{$\overline{\mathbf{U}}^k$}{$\overline{\mathbf{U}}^{k+1}$}{3}};
\node[inner sep=0.8mm, draw, fit=(A1)(A2)] (A) {};

\node[inner sep=0.1mm, above left=of A] (B1) {Owner \#1};
\node[draw, right=1mm of B1] (B2) {\passos{$\mathbf B_{A_1}^k$}{$\mathbf B_{A_1}^{k+1}$}{1}};
\node[inner sep=0.8mm, draw, dashed, fit=(B1)(B2)] (B) {};

\node[inner sep=0.1mm, left=of A] (C1) {Owner \#2};
\node[draw, right=1mm of C1] (C2) {\passos{$\mathbf B_{A_2}^k$}{$\mathbf B_{A_2}^{k+1}$}{1}};
\node[draw, dashed, fit=(C1)(C2)] (C) {};

\node[inner sep=0.1mm, below left=of A] (D1) {Owner \#3};
\node[draw, right=1mm of D1] (D2) {\passos{$\mathbf B_{A_3}^k$}{$\mathbf B_{A_3}^{k+1}$}{1}};
\node[inner sep=0.8mm, draw, dashed, fit=(D1)(D2)] (D) {};

\path[draw, latex-] ([yshift=-1.2mm]B.east) -| node[pos=0.25, below, font=\footnotesize] {$\overline{\mathbf H}^k-\overline{\mathbf{ZB}}^k-\mathbf U^k$} node[pos=0.12, bolanro]{4}([xshift=1mm]A.north west);
\path[draw, -latex] ([yshift=+1.2mm]B.east) -| node[pos=0.25, above, font=\footnotesize] {$\mathbf Z_{A_1}\mathbf B_{A_1}^{k+1}$} node[pos=0.35, bolanro]{2} ([xshift=1+2.4mm]A.north west);

\path[latex-] ([yshift=-1.2mm]C.east) edge node[below, font=\footnotesize] {$\overline{\mathbf H}^k-\overline{\mathbf{ZB}}^k-\mathbf U^k$} node[pos=0.25, bolanro]{4}([yshift=-1.2mm]A.west);
\path[-latex] ([yshift=+1.2mm]C.east) edge node[above, font=\footnotesize] {$\mathbf Z_{A_2}\mathbf B_{A_2}^{k+1}$} node[pos=0.77, bolanro]{2} ([yshift=+1.2mm]A.west);

\draw[latex-] ([yshift=-1.2mm]D.east)  -| node [pos=0.28, below, font=\footnotesize] {$\overline{\mathbf H}^k-\overline{\mathbf{ZB}}^k-\mathbf U^k$} node[pos=0.12, bolanro]{4} ([xshift=1+2.4mm]A.south west);
\path[draw, -latex] ([yshift=+1.2mm]D.east) -| node[pos=0.25, above, font=\footnotesize] {$\mathbf Z_{A_3}\mathbf B_{A_3}^{k+1}$} node[pos=0.37, bolanro]{2} ([xshift=1mm]A.south west);
\end{tikzpicture}
\caption{Distributed ADMM LASSO-VAR with a central node and 3 data owners {(related to the algorithm in~\eqref{eq:admmcolumn})}.}
\label{ADMMprivate}
\end{figure}
}

\begin{proof}
Using the notation of {S}ection~\ref{sec:admmVARLASSO}, each of the $n$ data owners is assumed to use the same number of lags $p$ to fit a LASSO-VAR model with a total number of $T$ records ({keep in mind} that $T>np$, otherwise there will be more coefficients to be determined than system equations). At the end of $k$ iterations, the central node receives a total of $T n k$ values from each data owner $i$, corresponding to $\mathbf Z_{A_i} \mathbf B_{A_i}^1,\mathbf Z_{A_i} \mathbf B_{A_i}^2, ..., \mathbf  Z_{A_i} \mathbf B_{A_i}^k {\in \mathbb{R}^{T\times n}}$, and does not know $pnk + Tp$, corresponding to $\mathbf B_{A_i}^1, ..., \mathbf B_{A_i}^k{\in \mathbb{R}^{p\times n}}$ and $ \mathbf Z_{A_i}{\in \mathbb{R}^{T\times p}}$, respectively, $i=1,...,n$. Given that, the solution of the inequality 
{
\begin{equation}
Tnk \geq pnk+Tp \ ,    
\end{equation}
}
in $k$, allows to infer that a confidentiality breach can occur at the end of 
\begin{equation}
k=\ceil[\bigg]{\frac{Tp}{Tn-pn}}    
\end{equation} 
iterations{, where $\ceil{x}$ denotes the ceiling function}. Since $T$ tends to be large, $k$ tends to $\ceil{p/n}$, which may represent a confidentiality breach if the number of iterations required for the algorithm to converge is greater than $\ceil{p/n}$.{\qed}

\end{proof}

{\begin{proposition}\label{prop:zhang2}
In the most optimistic scenario, without repeated values in $\mathbf{Y}_{A_i}\in\mathbb{R}^{T\times 1}$ and $\mathbf Z_{A_i}\in\mathbb{R}^{T\times p}$, when applying the algorithm from~\cite{zhang2019admm} to solve the LASSO-VAR model in~\eqref{eq:admmcolumn}, the data owners can recover the sensible data from competitors at the end of   \begin{equation}
k=\ceil[\bigg]{\frac{Tn+(n-1)(Tp+T)}{Tn-(n-1)pn}}    
\end{equation} 
iterations.
\end{proposition}
}

\begin{proof}
Without loss of generality, Owner~\#1 is considered the {semi-trusted} data owner {--- a semi-trusted data owner completes and shares his/her computations faithfully, but tries to learn additional information while or after the algorithm runs.} For each iteration $k$, this data owner receives the intermediate matrix $\overline{\mathbf H}^k-\underbrace{\overline{\mathbf{ZB}}^k}_{=\frac{1}{n}\sum_{i=1}^n \mathbf{Z}_{A_i}\mathbf B_{A_i}^k}-\mathbf U^k{\in \mathbb{R}^{T\times n}}$, which provides $T n$ values. However, Owner~\#1 does not know 
$$\underbrace{-\mathbf U^k+\overline{\mathbf H}^{k}}_{\in \mathbb{R}^{T \times n}},\underbrace{\mathbf{B}_{A_2}^k,\dots,\mathbf{B}_{A_n}^k}_\text{$n-1$ matrices $\in \mathbb{R}^{p\times n}$},\underbrace{\mathbf{Z}_{A_2},\dots,\mathbf{Z}_{A_n}}_\text{$n-1$ matrices $\in \mathbb{R}^{T\times p}$},\underbrace{\mathbf{Y}_{A_2},\dots,\mathbf{Y}_{A_n}}_\text{$n-1$ matrices $\in \mathbb{R}^{T\times 1}$},$$
which corresponds to $T n + (n-1) p n +(n-1)T p + (n-1)T$ values. However, since all the data owners know that $\overline{\mathbf{H}}^k$ and $\mathbf U^k$ are defined by the expressions in~\eqref{eq:admmcolumnB} and~\eqref{eq:admmcolumnC}, {it is possible to perform some} simplifications in {which} $\mathbf U^k$ and $\overline{\mathbf H}^k-\overline{\mathbf{ZB}}^k-\mathbf U^k$ becomes \eqref{eq:Usimplified} and  \eqref{eq:Msimplified}, respectively,
{\begin{equation}\label{eq:Usimplified}
\begin{split}
    \mathbf U^{k} \stackrel{\eqref{eq:admmcolumnC}}{=}&\mathbf{U}^{k-1}+\overline{\mathbf{ZB}}^k-\overline{\mathbf H}^k= \mathbf{U}^{k-1}+\overline{\mathbf{ZB}}^k-\underbrace{\frac{1}{N+\rho}\left(\mathbf Y+ \rho \overline{\mathbf{ZB}}^{k}+\rho \mathbf U^{k-1}\right)}_\text{$= \overline{\mathbf H}^k$, according to~\eqref{eq:admmcolumnB}}\\ =&\Big[1-\frac{\rho}{N+\rho} \Big] \mathbf U^{k-1} +  \Big[1-\frac{\rho}{N+\rho} \Big] \overline{\mathbf{ZB}}^{k} -\frac{1}{N+\rho} \mathbf Y \ ,
\end{split}
\end{equation}
}
{
\begin{equation}\label{eq:Msimplified}
\begin{split}
    \overline{\mathbf H}^k{-}\overline{\mathbf{ZB}}^k{-}\mathbf U^k &{=} \underbrace{\frac{1}{N+\rho}\left(\mathbf Y+ \rho \overline{\mathbf{ZB}}^{k}+\rho \mathbf U^{k-1}\right)}_\text{=$\overline{\mathbf{H}}^k$, according to~\eqref{eq:admmcolumnB}}-\overline{\mathbf{ZB}}^k-\mathbf U^{k} .
\end{split}
\end{equation}
}

Therefore, the iterative process to find the competitors' data proceeds as follows:

\begin{enumerate}
\item \textit{Initialization:} The central node generates $\mathbf U^0 \in \mathbb{R}^{T \times n}$, and the $i$-th data owner generates $\mathbf B_{A_i}^1 \in \mathbb{R}^{p \times n}$, $i \in \{1,...,n\}$. 

\item \textit{Iteration \#1:} The central node receives $\mathbf Z_{A_i} \mathbf B_{A_i}^1$ and computes $\mathbf U^1$, returning $\overline{\mathbf H}^1-\overline{\mathbf{ZB}}^1-\mathbf U^1 \in \mathbb{R}^{T \times n}$ which is returned for all $n$ data owners.  At this point, Owner~\#1 receives $T n$ values and does not know  
$$\underbrace{\mathbf U^0}_{\in \mathbb{R}^{T\times n}}, \underbrace{\mathbf B_{A_2}^1,...,\mathbf B_{A_n}^1}_\text{$n-1$ matrices $\in \mathbb{R}^{p\times n}$}, \underbrace{\mathbf Z_{A_2},...,\mathbf Z_{A_n}}_\text{$n-1$ matrices $\in \mathbb{R}^{T\times p}$} \ ,$$ and $n{-}1$ columns of $\mathbf Y\in \mathbb{R}^{T\times n}$, corresponding to $T n+ (n-1)[p n +T p+T]$ values.

\item \textit{Iteration \#2:} The central node receives $\mathbf Z_{A_i} \mathbf B_{A_i}^2$ and computes $\mathbf U^2$, returning $\overline{\mathbf H}^2-\overline{\mathbf{ZB}}^2-\mathbf U^2$ for the $n$ data owners.  At this point, only new estimations for the vectors $\mathbf B_{A_2},...,\mathbf B_{A_n}$ were introduced in the system, which means more $(n{-}1) p n$ values to estimate.
\end{enumerate}
As a result, at the end of $k$ iterations, Owner~\#1 has received {$\mathbf Z_{A_i}\mathbf B_{A_i}^1,\dots, \mathbf Z_{A_i}\mathbf B_{A_i}^k\in \mathbb{R}^{T\times n}$} corresponding to $T n k$ values and needs to estimate 
$$\underbrace{\mathbf U^0}_{\in \mathbb{R}^{T\times n}}, \underbrace{\mathbf B_{A_2}^1,...,\mathbf B_{A_n}^1, \mathbf B_{A_2}^2,...,\mathbf B_{A_n}^2,\dots, \mathbf B_{A_2}^k,...,\mathbf B_{A_n}^k}_\text{$(n-1)k$ matrices $\in \mathbb{R}^{p\times n}$}, \underbrace{\mathbf Z_{A_2},...,\mathbf Z_{A_n}}_\text{$n-1$ matrices $\in \mathbb{R}^{T\times p}$} \ ,$$ and $n{-}1$ columns of $\mathbf Y\in \mathbb{R}^{T\times n}$, corresponding to $T n {+} (n{-}1)[k p n {+}T p {+}T]$. Then, the solution for the inequality
\begin{equation}
Tnk \geq T n + (n-1)[k p n  +T p +T] { \ ,}    
\end{equation}  allows to infer that a confidentiality breach may occur at the end of
\begin{equation}
k=\ceil[\bigg]{\frac{Tn+(n-1)(Tp+T)}{Tn-(n-1)pn}}    
\end{equation} 
iterations. {\qed}
\end{proof}

Figure~\ref{fig:owner_fixedp} illustrates the $k$ value for different combinations of $T$, $n$, and $p$. In general, the greater the number of records $T$, the smaller the number of iterations necessary for confidentiality breach. That is because more information is shared during each iteration of the ADMM algorithm. On the other hand, the number of iterations until a possible confidentiality breach increases with the number of data owners ($n$). The same is true for the number of lags ($p$).

\begin{figure}
\centering
\includegraphics{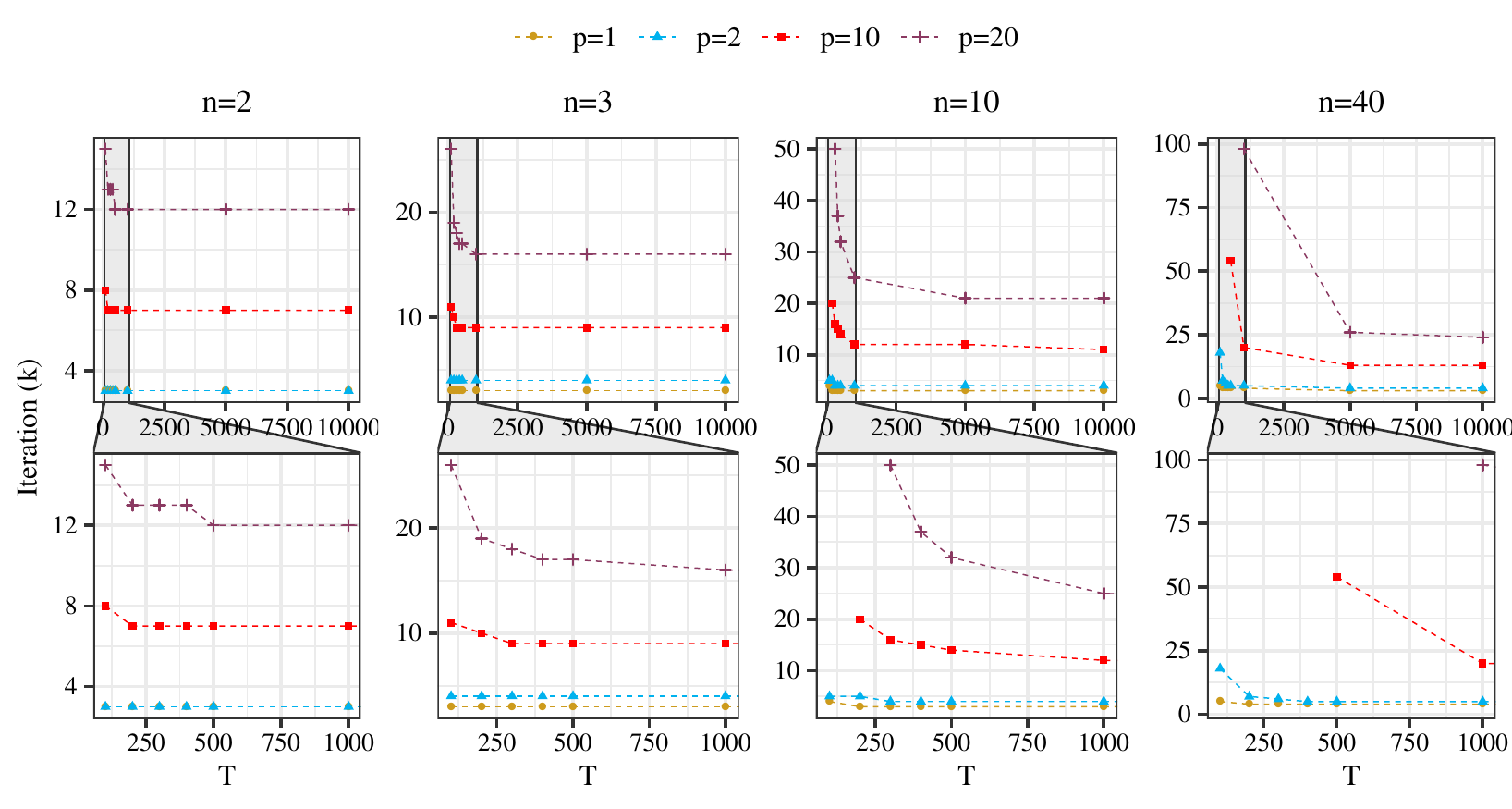}
\caption{{Number of iterations until a possible confidentiality breach, considering the centralized ADMM-based algorithm in~\citep{zhang2019admm}}.}
\label{fig:owner_fixedp}
\end{figure}

\subsubsection{ADMM Method and Noise Mechanisms}\label{subsec:admmnoiseAnalysis}

{The target matrix $\mathbf Y=[\mathbf Y_{A_1},\dots,\mathbf Y_{A_n}]$ corresponds to the sum of private matrices $\mathbf I_{\mathbf Y_{A_i}}\in \mathbb{R}^{T\times n}$, i.e.
\begin{equation}\label{eq:y_decompose}
\setlength{\tabcolsep}{2pt}
\underbrace{\left[\begin{tabular}{lllll}
$y_{1,t}$ & $y_{2,t}$ & \dots & $y_{n,t}$\\
$y_{1,t+1}$ & $y_{2,t+1}$ &\dots & $y_{n,t+1}$\\
$y_{1,t+2}$ & $y_{2,t+2}$ & \dots & $y_{n,t+2}$\\
\vdots &$\ddots$ & \vdots\\
$y_{1,t+h}$ & $y_{2,t+h}$ &\dots & $y_{n,t+h}$\\
\end{tabular}\right]}_{\mathbf Y}=
\underbrace{\left[\begin{tabular}{lllll}
$y_{1,t}$ & 0 & \dots & 0\\
$y_{1,t+1}$ & 0 &\dots & 0\\
$y_{1,t+2}$ & 0 & \dots & 0\\
\vdots &$\ddots$ & \vdots\\
$y_{1,t+h}$ & 0 &\dots & 0\\
\end{tabular}\right]}_{\mathbf I_{Y_{A_1}}}+
\underbrace{\left[\begin{tabular}{lllll}
0 & $y_{2,t}$  & \dots & 0\\
0 & $y_{1,t+1}$  &\dots & 0\\
0 & $y_{1,t+2}$  & \dots & 0\\
\vdots &$\ddots$ & \vdots\\
0 & $y_{1,t+h}$  &\dots & 0\\
\end{tabular}\right]}_{\mathbf I_{Y_{A_2}}}+ \dots+
\underbrace{\left[\begin{tabular}{lllll}
0 & 0 & \dots & $y_{n,t}$\\
0 & 0 &\dots & $y_{n,t+1}$\\
0 & 0 & \dots & $y_{n,t+2}$\\
\vdots &$\ddots$ & \vdots\\
0 & 0 &\dots & $y_{n,t+h}$\\
\end{tabular}\right]}_{\mathbf I_{Y_{A_n}}} \ ,
\end{equation}
where $[\mathbf I_{\mathbf Y_{A_i}}]_{i,j}{=}0$ in cases where the entry (i, j) of $\mathbf Y$ is from $i$-th data owner and $[\mathbf I_{\mathbf Y_{A_i}}]_{i,j}{=}[\mathbf Y]_{i,j}$ otherwise.}

Since the LASSO-VAR ADMM formulation is provided by~\eqref{eq:admmcolumn}, {at iteration $k$, data owners receive the intermediate matrix $\overline{\mathbf H}^k-\overline{\mathbf{ZB}}^k-\mathbf U^k$ and then update their local solution through~\eqref{eq:admmcolumnA}. The combination of~\eqref{eq:Usimplified} with~\eqref{eq:y_decompose} allows to rewrite $\mathbf U^{k}$ as
\begin{equation}\label{eq:u}
\mathbf U^k=\Big[1-\frac{\rho}{N+\rho} \Big] \mathbf U^{k-1} + \sum_{i=1}^n \underbrace{\Big[1-\frac{\rho}{N+\rho} \Big] \frac{1}{n} \mathbf Z_{A_i} \mathbf B_{A_i}^{k} -\frac{1}{N+\rho} \mathbf I_{\mathbf Y_{A_i}}}_\text{information from owner $i$} \ ,
\end{equation}
and, similarly, $\overline{\mathbf H}^k-\overline{\mathbf{ZB}}^k$ can be rewritten as
\begin{equation}\label{eq:hzb}
\begin{split}
\overline{\mathbf H}^{k} - \overline{\mathbf{ZB}}^k &=
\frac{1}{N+\rho}\mathbf Y+\Big[\frac{\rho}{N+\rho}-1\Big] \overline{\mathbf{ZB}}^k + \frac{\rho}{N+\rho} \mathbf U^{k-1} - \mathbf U^k \\&=
\sum_{i=1}^n \underbrace{\Big(\frac{1}{N+\rho}\mathbf I_{\mathbf Y_{A_i}}+\Big[\frac{\rho}{N+\rho}-1\Big]\frac{1}{n} \mathbf Z_{A_i} \mathbf B_{A_i}^k\Big)}_\text{information from owner $i$} + \frac{\rho}{N+\rho} \mathbf U^{k-1} - \mathbf U^k \ ,\\
\end{split}
\end{equation}}
where 
\begin{equation}
\mathbf Y=\sum_{i=1}^n \mathbf I_{\mathbf Y_{A_i}} \ ,    
\end{equation} 
\begin{equation}
\overline{\mathbf{ZB}}^{k+1}=\sum_{i=1}^n \frac{\rho}{n}\mathbf{Z}_{A_i}\mathbf B^{k+1}_{A_i} \ .   
\end{equation}  { By analyzing~\eqref{eq:u} and~\eqref{eq:hzb}, it is possible to verify that data owner $i$ only needs to share 
\begin{equation}\label{eq:int}
\frac{1}{N+\rho}\mathbf I_{\mathbf Y_{A_i}}+\Big[\frac{\rho}{N+\rho}-1\Big]\frac{1}{n} \mathbf Z_{A_i} \mathbf B_{A_i}^k \ ,
\end{equation}
for the computation of $\overline{\mathbf H}^k-\overline{\mathbf{ZB}}^k-\mathbf U^k$.
}

Let $\mathbf W_{1, A_i} \in \mathbb{R}^{T \times n}$, $\mathbf W_{2, A_i} \in \mathbb{R}^{T \times p}$, $\mathbf W_{3, A_i} \in \mathbb{R}^{p \times n}$, $\mathbf W_{4,  A_i} \in \mathbb{R}^{T \times n}$, represent noise matrices generated according to the differential privacy framework. The noise mechanism could be introduced by 
\begin{enumerate}[(i)]
\item adding noise to the data {itself}, i.e. replacing $\mathbf I_{\mathbf Y_{A_i}}$ and $\mathbf{Z}_{A_i}$ by
\begin{equation}\label{eq:noise_data}
    \mathbf I_{\mathbf Y_{A_i}}+\mathbf W_{1, A_i} \text{ and } \mathbf{Z}_{A_i}+\mathbf W_{2,A_i} \ ,
\end{equation} 
\item adding noise to the estimated coefficients, i.e. replacing $\mathbf B_{A_i}^k$ by
\begin{equation}\label{eq:noise_coefs}
    \mathbf{B}^{k}_{A_i}+\mathbf W_{3,  A_i} \ ,
\end{equation} 
\item adding noise to the intermediate matrix{~\eqref{eq:int}}, 
\begin{equation}\label{eq:noise_iterm}
   \frac{1}{N+\rho}\mathbf I_{\mathbf Y_{A_i}}+\Big[\frac{\rho}{N+\rho}-1\Big]\frac{1}{n} \mathbf Z_{A_i} \mathbf B_{A_i}^k+\mathbf W_{4,  A_i}.
\end{equation}
\end{enumerate}

{The addition of noise to the data itself~\eqref{eq:noise_data} was empirically analyzed in Subsection~\ref{subsec:noiseaddition_experiments} and, as verified, confidentiality comes at the cost of model accuracy deterioration.} The question is whether {adding noise to the coefficients or intermediate matrix} can ensure that data {are not recovered at the end of a number of iterations.
}

{
\begin{proposition}
Consider noise addition in an ADMM-based framework by 
\begin{enumerate}[(i)]
\item adding noise to the coefficients, as described in~\eqref{eq:noise_coefs};
\item adding noise to the exchanged intermediate matrix, as described in~\eqref{eq:noise_iterm}.

\end{enumerate}

Then, in both cases, a semi-trusted data owner can recover the data at the end of 
\begin{equation}
k=\ceil[\bigg]{\frac{Tn+(n-1)(Tp+T)}{Tn-(n-1)pn}}    
\end{equation} 
iterations.

\end{proposition}
}


\begin{proof} {These statements are promptly deduced from the Proof presented for Proposition~\ref{prop:zhang2}. Without loss of generality, Owner~\#1 is considered the semi-trusted data owner.} 

\begin{enumerate}[(i)]
\item {Owner~\#1} can estimate $\mathbf B_{A_i}$, without distinguishing between $\mathbf B_{A_i}$ and $\mathbf W_{3,A_i}$ in~\eqref{eq:noise_coefs}, by recovering $\mathbf I_{\mathbf Y_{A_i}}$ and $\mathbf{Z}_{A_i}$.
{Let $\mathbf B'_{A_i}=\mathbf B_{A_i}+\mathbf W_{3,A_i}$ and $\overline{\mathbf H'}^k$, $\mathbf U'^k$ be the matrices $\overline{\mathbf H}^k$, $\mathbf U^k$ replacing $\mathbf B'_{A_i}$ by $\mathbf B_{A_i}$.
Then, at iteration $k$ Owner~\#1 receives $\overline{\mathbf H'}^k-\overline{\mathbf{ZB'}}^k-\mathbf U'^k \in \mathbb{R}^{T\times n}$ ($Tn$ values) and does not know
$$\underbrace{\overline{\mathbf H'}^{k}-\mathbf U'^k}_{\in \mathbb{R}^{T \times n}},\underbrace{\mathbf{B'}_{A_2}^k,\dots,\mathbf{B'}_{A_n}^k}_\text{$n-1$ matrices $\in \mathbb{R}^{p\times n}$},\underbrace{\mathbf{Z}_{A_2},\dots,\mathbf{Z}_{A_n}}_\text{$n-1$ matrices $\in \mathbb{R}^{T\times p}$},\underbrace{\mathbf{Y}_{A_2},\dots,\mathbf{Y}_{A_n}}_\text{$n-1$ matrices $\in \mathbb{R}^{T\times 1}$},$$
which corresponds to $T n + (n-1) p n +(n-1)T p + (n-1)T$ values. As in Proposition~\ref{prop:zhang2}, this means that, after $k$ iterations, Owner~\#1 has received $T n k$ values and needs to estimate 
$$\underbrace{\mathbf U'^0}_{\in \mathbb{R}^{T\times n}}, \underbrace{\mathbf B'^1_{A_2},...,\mathbf B'^ 1_{A_n}, \mathbf B'^2_{A_2},...,\mathbf B'^2_{A_n},\dots, \mathbf B'^k_{A_2},...,\mathbf B'^k_{A_n}}_\text{$(n-1)k$ matrices $\in \mathbb{R}^{p\times n}$}, \underbrace{\mathbf Z_{A_2},...,\mathbf Z_{A_n}}_\text{$n-1$ matrices $\in \mathbb{R}^{T\times p}$} \ , $$ and $n{-}1$ columns of $\mathbf Y\in \mathbb{R}^{T\times n}$, corresponding to $T n {+} (n{-}1)[k p n {+}T p {+}T]$. Then, the solution for the inequality $Tnk \geq T n + (n-1)[k p n  +T p +T]$ allows to infer that a confidentiality breach may occur at the end of $$k=\ceil[\bigg]{\frac{Tn+(n-1)(Tp+T)}{Tn-(n-1)pn}}$$ iterations.}

\item {Since Owner~\#1 can estimate $\mathbf B_{A_i}$ by recovering data, adding noise to the intermediate matrix reduces to the case of adding noise to the coefficients, in (i), because Owner~\#1 can rewrite~\eqref{eq:noise_iterm} as 
\begin{equation}
\frac{1}{N+\rho}\mathbf I_{\mathbf Y_{A_i}}+\Big[\frac{\rho}{N+\rho}-1\Big]\frac{1}{n} \mathbf Z_{A_i} \Big[\underbrace{\mathbf B_{A_i}^k+\Big[\frac{\rho}{N+\rho}-1\Big]^{-1}\mathbf Z_{A_i}^{-1}\mathbf W_{4,A_i}}_{=\mathbf B'_{A_i}}\Big].
\end{equation}
{\qed}
}
\end{enumerate}

\end{proof}

\section{Discussion}
\label{sec:discussion}

\begin{table}
\centering
\caption{Summary of state-of-the-art privacy-preserving approaches.}
\label{tab:approaches}
\begin{tabular}{p{10.8em}p{6em}p{8.4em}p{8.4em}}
\hline
& & \bf Split by features & \bf Split by records\\
\hline
\bf Data Transformation &  &  \small\cite{mangasarian2011privacy}       & \small\cite{mangasarian2012privacy}, \cite{yu2008privacy}, \cite{dwork2014analyze}  \\
\hline
\bf \multirow{2}{10.8em}{Secure Multi-party Computation} & Linear\newline Algebra       &  \small\cite{du2004privacy}, \cite{karr2009privacy}, \cite{zhu2015privacy}, \cite{fan2014adaptive}*, \cite{soria2017individual} & \small\cite{zhu2015privacy}, \cite{aono2017input}  \\ \cline{2-4}
& Homomorphic-cryptography &   \small\cite{yang2019federated}, \cite{hall2011secure}, \cite{gascon2017privacy}, \cite{slavkovic2007secure}       &  \small\cite{yang2019federated}, \cite{hall2011secure}, \cite{nikolaenko2013privacy}, \cite{chen2018privacy}, \cite{jia2018preserving}, \cite{slavkovic2007secure}  \\
\hline
\bf \multirow{3}{10.8em}{Decomposition-based Methods} & Pure & \small\cite{pinson2016introducing}, \cite{zhang2018distributed} &    \small\cite{wu2012g}, \cite{lu2015webdisco}, \cite{ahmadi2010privacy}, \cite{mateos2010distributed} \\\cline{2-4}
& Linear\newline Algebra & \small\cite{li2015vertical}, \cite{han2010privacy} & \small\cite{zhang2017dynamic}, \cite{huang2018dp}, \cite{zhang2018recycled}\\\cline{2-4}
& Homomorphic-cryptography & \small\cite{yang2019federated}, \cite{li2012efficient}*, \cite{liu2018practical}*, \cite{li2018ppma}*, \cite{fienberg2009valid}, \cite{mohassel2017secureml} & \small\cite{yang2019federated}, \cite{zhang2019admm}, \cite{fienberg2009valid}, \cite{mohassel2017secureml}  \\
\hline
\multicolumn{4}{l}{\scriptsize * secure data aggregation.}\\
\end{tabular}
\end{table}

Table~\ref{tab:approaches} summarizes {the} methods from the literature. 
These algorithms for {the} privacy-preserving {ought} to be carefully {built and consider} two {key} components: (i)~how data is distributed between data owners, and (ii)~the statistical model used. The decomposition-based methods are very sensitive to data partition, while data transformation and cryptography-based methods are very sensitive to problem structure, with the exception of the differential privacy methods, which simply add random noise, from specific probability distributions, to the data itself. This property makes these methods appealing, but differential privacy usually involves a trade-off between accuracy and privacy.

Cryptography-based methods are usually more {effective against} confidentiality {breaches}, but {they have} some disadvantages:
(i)~some of them require a third-party for keys generation, as well as external entities to perform the computations in the encrypted domain; (ii)~challenges in the scalability and implementation efficiency, which {are} mostly due to the high computational complexity and overhead of existing homomorphic encryption schemes~\citep{de2012design, zhao2019secure, tran2019privacy}. {Regarding} some protocols, such as secure {multi-party} computation {through homomorphic cryptography}, communication complexity grows exponentially with the number of records~\citep{rathore2015survey}.

{Data transformation methods do not affect the computational time for training the model, since each data owner transforms his/her data before the model fitting process. The same is true for the decomposition-based methods in which data is split by data owners. The secure multi-party protocols have the disadvantage of transforming the information while fitting the statistical model, which implies a higher computational cost.}

As already mentioned, the {main} challenge {to the application of} the existent privacy-preserving algorithms in the VAR model is the fact that $\mathbf Y$ and $\mathbf Z$ share a high percentage of values, not only during the fitting of the statistical model but also when using it to perform forecasts. {A confidentiality breach may occur during the forecasting process if, after the model is estimated, the algorithm to maintain privacy provides the coefficient matrix $\mathbf B$ for all data owners}. When using the estimated model to perform forecasts, assuming that each $i$-th data owner sends their own contribution for the time series forecasting to every other $j$-th data owner:

\begin{enumerate}
\item In the LASSO-VAR models with one lag, since $i$-th data owner sends $y_{i,t} [\mathbf B^{(1)}]_{i,j}$ for $j$-th data owner, {the value $y_{i,t}$ may be directly recovered when the coefficient $[\mathbf B^{(1)}]_{i,j}$ is known by all data owners}, being $[\mathbf B^{(1)}]_{i,j}$ the coefficient associated with lag $1$ of time series $i$, to estimate $j$.
\item In the LASSO-VAR models with $p$ consecutive lags, the forecasting of a new timestamp only requires the introduction of one new value in the covariate matrix of the $i$-th data owner. {In other words}, at the end of $h$ timestamps, the $j$-th data owner {receives the} $h$ values. {However, t}here are $h+p$ values that {the data owner} does not know about. This may represent a confidentiality breach since a {semi-trusted} data owner can assume different possibilities for the initial $p$ values and then generate possible trajectories.
\item In the LASSO-VAR models with $p$ non-consecutive lags, $p_1,\dots,p_p$, at the end of $p_p-p_{p-1}$ timestamps, only one new value is introduced in the covariate matrix, meaning that the model is also subject to a {confidentiality breach}.
\end{enumerate}

Therefore, {and considering the issue of data} naturally split by features, it would be more advantageous to apply decomposition-based methods, since the time required for model fitting is not affected by data transformations and each data owner only has access to their own coefficients. However, with the state-of-the-art approaches, it is difficult to guarantee that these techniques can indeed offer a robust solution for data privacy {when addressing} data split by features.  
Finally, a remark about some specific business applications of VAR, where data owners know exactly some past values of the competitors.
For example, consider a VAR model with lags $\Delta t=1$, 2 and 24, which predicts the production of solar plants. Then, when forecasting the first sunlight hour of a day, all data owners will know that the previous lags 1 and 2 have zero production (no sunlight). Irrespective of whether the coefficients are shared or not, a confidentiality breach may occur. For these special cases, the estimated coefficients cannot be used for a long time horizon, and online learning may represent an efficient alternative.

The privacy issues analyzed in this paper are not restricted to the VAR
model nor to point forecasting tasks. Probabilistic forecasts, using data from different data owners (or geographical locations), can be generated with splines quantile regression~\citep{tastu2013probabilistic}, component-wise gradient boosting~\citep{bessa2015probabilistic}, a VAR that estimates the location parameter (mean) of data transformed by a logit-normal distribution~\citep{dowell2015very}, {linear} quantile regression with LASSO regularization~\citep{agoua2018probabilistic}, among others. These are {some} examples of collaborative probabilistic forecasting methods. However, none of them considers the confidentiality of the data. {Moreover}, the method proposed by~\cite{dowell2015very} {can be influenced by} the confidentiality breaches discussed thorough this paper, since the VAR model is directly used to estimate the mean of transformed data from the different data owners. On the other hand, when performing non-parametric models such as quantile regression, each quantile is estimated by solving an independent optimization problem, which means that the risk of a confidentiality breach increases with the number of quantiles being estimated. {Note that quantile regression-based models may be solved through the ADMM method~\citep{zhang2019admm}. However, as discussed in {S}ection~\ref{sec:decomp}, the {semi-trusted} agent may collect enough information to infer the confidential data. The quantile regression method may also be estimated by applying linear programming algorithms~\citep{agoua2018probabilistic}, which may be solved through homomorphic encryption, {despite} being computationally demanding for high-dimensional multivariate time series.}

\section{Conclusion}
\label{sec:conclusion}

{This paper presents a} critical overview of the literature techniques used to handle privacy issues in {collaborative forecasting} methods. {In addition, it also performs an analysis to} their application to the VAR model. The {aforementioned} existing techniques {are} divided into three {groups} of approaches to guarantee privacy: data transformation, secure multi-party computation and decomposition of the optimization problem into sub-problems.

For each {group}, several points can be concluded. Starting with \textit{data transformation techniques}, two remarks were made. The first one concerns the addition of random noise to the data. While the algorithm is simple to apply, this technique demands a trade-off between privacy and the correct estimation of the model's parameters{~\citep{yang2019federated}}. In our experiments, there was {a clear} model degradation even though the data kept its original behavior {({S}ection~\ref{subsec:noiseaddition_experiments})}. The second relates to the multiplication by a random matrix that is kept undisclosed. {Ideally,  and in what concerns} data where different data owners observe different variables, this secret matrix would post-multiply data, {thus enabling} each data owner to generate a few lines of this matrix. However, as demonstrated in equation~\eqref{eq:post_multipication} of {S}ection~\ref{sec:post_multiplication_multiple}, this transformation does not preserve the {estimated coefficients}, and the reconstruction of the original model may require sharing the matrices used to encrypt the data, thus exposing the original data.

The second group of techniques, \textit{secure multi-party computations}, introduce privacy in the intermediate computations by defining the protocols for addition and multiplication of the private datasets{, without confidentiality breaches using either linear algebra or homomorphic encryption methods.} {For independent records, data confidentiality is guaranteed for (ridge) linear regression through linear algebra-based protocols; not only do records need to be independent, but some also require that the target variable is known by all data owners. These assumptions might prevent their application when covariates and target matrices share a large proportion of values -- in the VAR model case, for instance.} This means that shared data between agents might be enough for competitors to be able to reconstruct the data. Homomorphic cryptography methods might result in computationally demanding techniques since each dataset value has to be encrypted. {The discussed protocols ensure privacy-preserving while using (ridge) linear regression if there are two entities that correctly performs the protocol without agent collusion. These entities are an external server (e.g., a cloud server) and an entity which generates the encryption keys. In some approaches, all data owners know the coefficient matrix $\mathbf B$ at the end of the model estimation. This is a disadvantage when applying models in which covariates include the lags of the target variable because confidentiality breaches may occur during the forecasting phase.}

Finally, {the} \textit{decomposition of the optimization problem} into sub-problems {(which} can be solved in parallel{)} have all the desired properties for a collaborative forecasting problem, since each data owner only estimates their coefficients.  {A common assumption of such methods is that the objective function is decomposable.} However, these approaches consist in iterative processes that require sharing intermediate results for the next update, meaning that each new iteration conveys more information about the secret datasets to the data owners, possibly breaching data confidentiality.


\section*{Acknowledgements}

The research leading to this work is being carried out as part of the Smart4RES project (European
Union’s Horizon 2020, No. 864337). Carla Gon\c{c}alves was supported by the Portuguese funding agency,
FCT (Funda\c{c}\~ao para a Ci\^encia e a Tecnologia), within the Ph.D. grant PD/BD/128189/2016 with financing
from POCH (Operational Program of Human Capital) and the EU. The sole responsibility for the content
lies with the authors. It does not necessarily reflect the opinion of the Innovation and Networks Executive
Agency (INEA) or the European Commission (EC), which are not responsible for any use that may be made
of the information it contains.

\appendix

\section{Differential Privacy}\label{diff-privacy}
Mathematically, a randomized mechanism $\mathcal A$ satisfies ($\varepsilon$,$\delta$)-differential privacy{~\citep{dwork2009differential}} if, for every possible output $t$ of $\mathcal A$  and for every pair of datasets $\mathbf D$ and $ \mathbf D'$ (differing in at most one record), 
\begin{equation}
    \mbox{Pr}(\mathcal A( \mathbf D) = t) \leq \delta + \exp(\varepsilon) \mbox{Pr}(\mathcal A( \mathbf D') = t).
\end{equation}
In practice, differential privacy can be achieved by adding random noise  $W$ to some desirable function $f$ of the data $\mathbf D$, i.e 
\begin{equation}
    \mathcal A(\mathbf D) = f(\mathbf D) +W. 
\end{equation}
The ($\varepsilon$,0)-differential privacy is achieved by applying noise from Laplace distribution with scale parameter $\frac{\Delta f_1}{\varepsilon}$, with $\Delta f_k = \max\{\|f(\mathbf D) - f(\mathbf D')\|_{k}\}$. A common alternative is the Gaussian distribution but, in this case, $\delta>0$ and the scale parameter which allows ($\varepsilon$,$\delta$)-differential privacy is $\sigma \geq \sqrt{2\log\Big(\frac{1.25}{\delta}\Big)} \frac{\Delta_2 f}{\varepsilon}$. {\cite{dwork2009differential} showed that the data can be masked by considering}
\begin{equation}
    \mathcal A(\mathbf D) = \mathbf D+\mathbf W.
\end{equation}

{
\section{Supplementary Data and Code}\label{app:data_code}
Supplementary material related to this article is available online (\url{https://doi.org/10.25747/gywm-9457}). The available material includes: 
\begin{itemize}
\item \texttt{admm\_functions.R}: R script with ADMM algorithm implementation.
\item \texttt{clear\_sky\_functions.R}: R script to estimate clear-sky solar power generation with the model described in~\cite{Bacher2009}.
\item \texttt{coef\_generator.R}: R script with the functions for generating VAR model coefficients, according to the implementation in~\citep{gmvarkit2020}.
\item \texttt{run\_experiments.R}: R script with the commands for generating the results of {S}ection~\ref{subsec:noiseaddition_experiments}.
\item \texttt{c\_sky.csv}: estimated clear-sky solar power generation. 
\item \texttt{normalized\_PVdata.csv}: normalized (with clear-sky model) solar power time series data. 
\item \texttt{PVdata.csv}: solar power time series data. 
\end{itemize}
}

\bibliography{library}

\end{document}